%% file: ddpm_based_tilting.tex
\documentclass[a4paper, 12pt]{amsart}
\usepackage{graphicx} 
\usepackage{float}
\usepackage{cancel}
\graphicspath{{plots/}}
\usepackage{amsmath, amssymb, amsfonts,amsthm,comment,enumitem, mathrsfs, algorithm, algpseudocode, amsaddr}
\usepackage[dvipsnames, svgnames]{xcolor}

\usepackage{algorithm}
\usepackage{algpseudocode}

\newtheorem{theorem}{Theorem}
\newtheorem*{theorem*}{Theorem}
\newtheorem{corollary}{Corollary}
\newtheorem{lemma}{Lemma}
\newtheorem{proposition}{Proposition}

\newtheorem{assumption}{Assumption}
\newtheorem{nota}{Notation}

\newcommand{\ZZ}{{\mathbf{Z}}}
\newcommand{\poubelle}[1]{}

\newcommand{\e}{{\varepsilon}}

\newcommand{\E}{\mathbb{E}}

\newcommand{\Var}{\mathbb{V}{\rm ar}\,}
\newcommand{\Cov}{\mathbb{C}{\rm ov}\,}

\newcommand{\cW}{{\mathcal W}}
\newcommand{\cD}{{\mathcal D}}

\newcommand{\cT}{{\mathcal T}}
\newcommand{\cP}{{\mathcal P}}
\newcommand{\cR}{{\mathcal R}}

\newcommand{\rx}{{\displaystyle{x}}}
\include{math_commands}

\usepackage{hyperref}
\usepackage{url}


\title{Generating DDPM-based Samples from Tilted Distributions}

\author{\centering Himadri Mandal$^1$, Dhruman Gupta$^{2}$, Rushil Gupta$^{2}$\\ Sarvesh R. Iyer$^{2}$, Agniv Bandyopadhyay$^3$\\ Achal Bassamboo$^4$, Varun Gupta$^5$, Sandeep Juneja$^{2}$}
\thanks{ $1 : $ Indian Statistical Institute, Kolkata. ${2} : $ Ashoka University. $3 : $ Tata Institute of Fundamental Research, $4 : $ Northwestern University, $5 : $ The University of Utah}
\thanks{Note : The three faculty names are listed at the end and arranged in alphabetical order by last name. The first author is the corresponding author.}

\begin{document}

\maketitle

\begin{abstract}

Given $n$ independent samples from a $d$-dimensional probability distribution, our aim is to generate diffusion-based samples from a distribution obtained by tilting the original, where the degree of tilt is parametrized by $\theta \in \R^d$. We define a plug-in estimator and show that it is minimax-optimal. We develop Wasserstein bounds between the distribution of the plug-in estimator and the true distribution as a function of $n$ and $\theta$, illustrating regimes where the output and the desired true distribution are close. Further, under some assumptions, we prove the TV-accuracy of running Diffusion on these tilted samples. Our theoretical results are supported by extensive simulations. Applications of our work include finance, weather and climate modelling, and many other domains, where the aim may be to generate samples from a tilted distribution that satisfies practically motivated moment constraints.

\end{abstract}

\section{INTRODUCTION}

Diffusion based 
generative-AI typically considers using samples
from a high-dimensional distribution to generate more independent samples from the same distribution. Traditionally, samples corresponded
to images or text, although recently there has been research on
the underlying samples corresponding to high-dimensional data in weather or climate modeling \cite{Li2024SEEDS,Ling2024DPDM} or in finance (see \cite{cont2025tail} for GAN based applications).

However, in many practically important settings, samples may be available
under one distribution $\mu(x)$ and our interest may be in generating samples
from a related \emph{tilted} distribution $\nu(x)\propto \exp(\theta^{T}\rg(\rx))\mu(x)$ (that is, a distribution obtained by multiplicatively biasing the underlying distribution appropriately).
 This is relevant in portfolio optimization \cite{meucci2010fullyflexibleviewstheory} and in option pricing in finance \cite{buchen1996maximum,stutzermichaelnonparametricapproach,avellaneda1998minimum}, where exponentially tilted versions of pricing distributions are sought to better model asset prices, and the exponent in the tilt is a linear function of stock returns. \cite{GollRueschendorf2002} consider exponentially tilted distributions
in mathematical finance where the exponent may be a nonlinear function of financial securities. 

Tilting of random variables or vectors is a common technique in the field of rare event sampling and Monte Carlo simulation, that effects an exponential change on an underlying probability measure in order to bias its rare events into becoming more common. It also has applications \cite{op} to phase transitions in physical phenomena \cite{lee2025exponentiallytiltedthermodynamicmaps}, rare event sampling from discrete Markov processes \cite{PhysRevE.109.034113}, and finding robust best response strategies in security games \cite{kong2025robust}. We refer the reader to \cite[Chapter 1]{Alvo2022} for an excellent review on applications of exponential tilting. 

Tilted distributions arise as a solution to a well-motivated optimization problem. For example, given a probability measure $\mu$ on $\R^d$, if we look for a probability measure $\nu$ that is closest to $\mu$ in relative entropy, and satisfies certain moment constraints, the solution is a distribution 
obtained by appropriately exponentially tilting $\mu$  with a  linear term in the exponent. When more general f-divergences are used instead of relative entropy, the solution can be an exponentially tilted distribution with a non-linear term in the exponent\cite{csiszar1967, csiszar2008axiomatic}. 

Previous works have largely approached tilted or biased sampling through importance sampling or MCMC biasing schemes, but these typically require some knowledge of the target density. More recently, diffusion-based methods have been proposed to enable sampling from tilted distributions via score guidance methods (see \cite{wang2024proteinconformationgenerationforceguided} and the references therein).

In this article, we take samples from an unknown high-dimensional probability distribution $\mu$. Our aim is to establish conditions under which a diffusion sampling-based algorithm is capable of producing high-quality samples from $\mu$ exponentially tilted by any desired vector $\theta$, i.e., $\nu(x) \propto \exp(\theta^T \rg(\rx)) \mu(x)$, where $\rg(\rx)$ is some desirable tilting function.

The proposed algorithm proceeds in two steps. 
\begin{enumerate}
\item Given samples $\{X_i\}$ from $\mu$, we first reweigh them using $w_i = \exp(\theta^T \rg(X_i))$, normalized by their sum $\sum_{i=1}^n w_i$. Sampling under the reweighed scheme yields i.i.d. samples whose distributional proximity to the tilted distribution is desired. Further, we must justify the choice of our scheme over other schemes, which is achieved by the minimax result Theorem~\ref{thm-minimax}. 
\item Then, we perform diffusion sampling by training a denoiser using the reweighed samples and running the reverse diffusion to generate new samples, which are expected to arise from a distribution close to the tilted distribution as well.
\end{enumerate}
The accuracy of the above algorithm is asserted by controlling the key quantities in each of the steps, namely,
\begin{enumerate}
\item the Wasserstein precision of the reweighed samples in comparison to empirical samples from the true twisted distribution as a function of $\theta$, $n$, and $\mu$, and
\item the TV-accuracy of diffusion sampling if input samples do not come from the true distribution but from a Wasserstein-nearby one.
\end{enumerate}

Previous literature that attempts to address our problem shows that while guidance methods such as diffusion posterior sampling \cite{chung2024diffusionposteriorsamplinggeneral} and loss-guided diffusion \cite{pmlr-v202-song23k} assume a differentiable tilting function and use alternate algorithms to produce tilted samples by diffusion, they typically rely on heuristic approximations of the score update. Consequently, theoretical support for these frameworks remains limited. Establishing rigorous guaranties is crucial, particularly in high-stakes applications such as financial risk assessment and weather forecasting. We shall now briefly talk about the literature that covers each of the questions, but not both together.

 The plug-in estimator used in the first step is an example of a self-normalized importance sampler \cite{Owen2000_SafeEffectiveImportanceSampling} or a weighted importance sampler \cite{Rubinstein1981_SimulationAndTheMonteCarloMethod}. To our knowledge, the accuracy of the plug-in estimator in the Wasserstein distance has not been studied, although \cite{iyer2025fundamental} considers the question for the KS distance and distributions with well-behaved tail, while \cite{chatterjee2018sample_size_importance} considers it in the $L^1$ distance. The question of its minimaxity in any sense has not been considered before.
 
 However, the second question has only been studied in the typical setting, with \cite{chen2023samplingeasylearningscore, benton2023linear_arxiv} (see also the references therein) showing the accuracy of vanilla diffusion (i.e. DDPM) in the TV distance, highlighting that controlling the \emph{denoiser error} is key to diffusion accuracy. In addition, \cite{accumulation_score_error_iclr2026,gao2025wasserstein} studies its accuracy in the Wasserstein distance. The former work, however, employs a log-concavity assumption, while the latter does not consider an initial perturbation of the data.

In this article, we positively address both questions, thus demonstrating the accuracy of tilted samples produced by diffusion. In Section~\ref{sec:reweigh}, we motivate the notion of tilted distributions as entropic distance minimizers to the original distribution under standard constraints, and subsequently address the first question. Under a Lipschitz assumption, we then establish denoiser error bounds in Section~\ref{sec:diff}, following which the answer to the second question is obtained directly from \cite{chen2023samplingeasylearningscore}. We end the article with extensive experimentation in Section~\ref{sec:exp}. The proofs of all our results can be found in Appendix~\ref{app:proofs}, while the experiments are in Appendix~\ref{app:experiments}.

\section{ON THE ACCURACY OF REWEIGHED SAMPLES}\label{sec:reweigh}

Throughout this article, let $\rx$ be a random vector in $\R^d$ with its multivariate CDF given by $F$. In Section~\ref{rew}, we first define the tilting and the plug-in estimator and state our first result, which is the minimaxity of this estimator. In Section~\ref{wacre} we assert the Wasserstein accuracy of the estimator, which is our second main result.

\subsection{THE PLUG-IN ESTIMATOR AND ITS MINIMAXITY}\label{rew}

Given $\theta \in \R^d$, an \emph{exponential tilting} of $\rx$ by the vector $\theta$ produces a random vector $\rx_{\theta}$ whose distribution is given by
\begin{equation*}
\mathbb P(\rx_{\theta} \in A) = \frac{\E[\exp(\theta^T \rg(\rx))\1_{\rx \in A}]}{\E[\exp(\theta^T \rg(\rx))]}
\end{equation*}
for some suitable function $g$ which can be considered a site-specific tilt. For notational convenience, let \begin{equation}\label{mtheta}M(\theta,A) = \E[\exp(\theta^T \rg(\rx))\1_{\rx \in A}], \quad M(\theta) = M(\theta,\R),\end{equation}
so that 
\begin{equation}\label{true}
\mathbb P(\rx_{\theta} \in A) = \frac{M(\theta,A)}{M(\theta)}.
\end{equation}

 As mentioned in the previous section, tilting is motivated by entropic distance minimization subject to mean constraints. For instance, we have the following result from \cite{csiszar2008axiomatic}.

\begin{proposition}[KL $\Rightarrow$ exponential tilt]\label{thm:KLtoExp}
Let $\Phi$ be a random variable on a measurable space $(\Omega, \mathcal{F})$. Fix a probability measure $Q$ such that $\E_{Q}[\exp(\Phi)]< \infty$, and for every probability measure $P$ such that $P \ll Q$ and $\E_P[\Phi]<\infty$, consider the functional
\[
\mathcal{J}_{\mathrm{KL}}(P)=\mathbb{E}_P[\Phi]-D_{\mathrm{KL}}(P\|Q),
\]
where $D_{\mathrm{KL}}$ denotes the $\mathrm{KL}$ divergence. Then, $\mathcal{J}_{\mathrm{KL}}$ is maximized uniquely at $P^\star$, where
\[
P^\star(A)=\frac{\E_{Q}[\exp(\Phi)\1_{A}]}{\E_{Q}[\exp(\Phi)]}.
\]
That is, $P^{\star}$ is a tilt of $Q$.
\end{proposition}

We motivate further tilting functions in Appendix~\ref{sec:entdiv}.

Let $F_\theta$ be the CDF of $\rx_{\theta}$ and $\mu_\theta$ be the probability measure corresponding to $\rx_{\theta}$. Given independent, identically distributed samples $\rx_1, \rx_2, \dots, \rx_n$ from $F$, we wish to use these samples to generate new independent samples from $F_{\theta}$, or equivalently approximate the CDF $F$ itself.

One natural way is as follows: approximating the RHS of \eqref{true} by replacing each expectation with their empirical estimators, we obtain a new random CDF (since it depends on the samples $\rx_i$), which is an estimator of the true CDF. We refer to this as the \emph{plug-in estimator}, since it was obtained by plugging in empirical estimators. 

To define the estimator, let $w = \exp(\theta^T g(x))$ be the distribution of the "weight" associated to $\rx$, and $w_i = \exp(\theta^T \rg(\rx_i))$ denote the weights of the samples, so that $\E[w] =\E[w_i] = M(\theta)$. Consider the measure
\begin{equation}\label{emp}
    \mu_{n, \theta} = \frac{1}{n} \sum_{i=1}^{n} \frac{w_i}{w^*_n} \delta_{\rx_{i}}, \quad \text{ where } w^{*}_n = \frac 1n\sum_{j=1}^n w_j.
\end{equation}

Note that $\mu_{n,\theta}$ is a linear combination of the $\delta_{\rx_i}$, and therefore can be considered as a weighted sampler, with $w_i$ being the weights and $w_n^*$ their average. Denote the corresponding CDF by $F_{n, \theta}$. Note that $\mu_{n, \theta}$ is a random point measure, or more precisely a point process in $\mathbb R^d$. 

The next theorem proves that in fact, the plug-in estimator in \eqref{emp} is asymptotically minimax for 1D distributions, which we consider in the sense of ~\cite{dvoretzky1956asymptotic}. The proof can be extended to the multidimensional case directly. Informally, this result states that \eqref{emp} is asymptotically the best estimator of $F_{\theta}$, which justifies further study of its behavior.

To state our result, let $\mathcal{F} = \mathcal{F}^B_{m, M}$ be the class of CDFs with a bounded support $\{\|x\| \leq B\}$ and $m \leq \E_{X \sim F}\exp(\theta X) \leq M$ Then, 

\begin{theorem}[Asymptotic Minimaxity]\label{thm-minimax}
    For every value $r > 0$, 
    $$
        \lim_{n \to \infty} 
        \frac{ \underset{F \in \mathcal{F}}{\sup} P_F \left\{ \sup_x|F_{n, \theta}(x) - F_\theta(x)| > \frac{r}{\sqrt{n}}\right\}}
        {\underset{\phi \in D_n}{\inf}\underset{F \in \mathcal{F}}{\sup} P_F \left\{\sup_x|\rg(\rx) - F_\theta(x)| > \frac{r}{\sqrt{n}} \right\}} = 1.
    $$
\end{theorem}
This demonstrates why the plug-in estimator is the natural, ``correct" choice. We prove this in the Appendix~\ref{app:sec-minimax}. 

Some comments on the proof are below. To prove this, one starts by proving that the $\lim$ and $\sup$ can be interchanged. 

Consider CDFs supported on finitely many points parametrized by $\pi \in \Delta^k$. The idea is to use local asymptotic minimaxity of the plug-in estimator in this sub-model. This follows since the finite support sub-model parametrized by $\pi$ admits a Local Asymptotic Normality (LAN) expansion
(See \cite[Chapter 8]{Vaart_1998}).

Now, we find a cdf $F^*$ that is $\varepsilon-$close to the probability of maximum deviation in the numerator. Then, the finite supports CDFs are used to grid $\mathcal{F}$, and hence used to approximate $F^{*}$. The details are much more complicated than the minimaxity of the empirical estimator, as the numerator depends on $F$ (which is not the case in the empirical estimator). 

Having justified the importance of the plug-in estimator, we proceed to study its Wasserstein accuracy.

\subsection{WASSERSTEIN ACCURACY OF THE PLUG-IN ESTIMATOR}\label{wacre}

Let $\rx_1,\ldots,\rx_n$ be identically distributed samples from $F_\theta$, and define $\mu_{n, \theta}^e$ to be the i.i.d empirical distribution of $F_\theta$. The following quantities will be key to the bounds we obtain :
\begin{gather}
M_q(\mu) = \E_{\rx \sim \mu}\|\rx\|^q,\label{Mq}\\
W_k = \E\left[\frac{M(k \theta)^{\frac 1k}}{M(\frac k2 \theta)^{\frac 2k}}\right], \label{Wk} \\
C_w = M(-2\theta)M(2\theta).\label{Cw}
\end{gather}

Note that $M_q(\mu)$ are the moments of $\mu$, while $W_k$ and $C_w$ are quantities that originate from analyzing the empirical error estimation of $M(\theta)$ by $w_n^{*}$ in \eqref{emp}. 

In order to state our next result, we define the $p$-Wasserstein distances $W_p$ for $p \geq 1$ as follows. For any measures $\nu_i, i=1,2$ on possibly different probability spaces $(\Omega_i,\mathcal F_i), i=1,2$, we call a probability measure $\pi$ on $(\Omega_1 \times \Omega_2, \mathcal F_1 \otimes \mathcal F_2)$ a \emph{coupling} of $\nu_1$ and $\nu_2$, if $\nu_1$ and $\nu_2$ are the marginals of $\pi$( i.e. $\pi(\Omega_1 \times \cdot) = \nu_2(\cdot)$ and $\pi(\cdot \times \Omega_2) = \nu_1(\cdot)$).  Let $\Pi(\nu_1,\nu_2)$ be the set of all couplings of $\nu_1$ and $\nu_2$. Define for $p \geq 1$, 
\begin{equation}\label{wass}
\cW_p(\nu_1,\nu_2) = \inf_{\pi \in \Pi(\nu_1,\nu_2)} \left(\mathbb E_{\rx \sim \nu_1,\ry \sim \nu_2} \|\rx-\ry\|^p\right)^{\frac 1p}.
\end{equation}

and 
\begin{equation}
\cT_p(\nu_1,\nu_2) = \inf_{\pi \in \Pi(\nu_1,\nu_2)} \mathbb E_{\rx \sim \nu_1,\ry \sim \nu_2} \|\rx-\ry\|^p.
\end{equation}
The importance of the Wasserstein distances to our analysis is paramount. Indeed, the proof of  Proposition~\ref{thm:diffbds} strongly uses the idea of couplings. Observe that $\cT_p = \cW_p^p$, and therefore $\E[\cW_p(\mu,\nu)] \leq \E[\cT_p(\mu,\nu)]^{\frac 1p}$ for all $p \geq 1$ and measures $\mu,\nu$ by Jensen's inequality. Therefore, while all our upper bounds will be on the quantity $\E[\cT_p(\mu,\nu)]$, they can easily be adjusted to $\E[\cW_p(\mu,\nu)]$, including asymptotic ones. Furthermore, observe that this relationship allows us to estimate moments of $W_p$, thereby describing a family of concentration inequalities for $W_p(\mu,\nu)$ as well (which we shall not elaborate on). Thus, we continue to describe our estimates as being in the expected Wasserstein distance.

Before presenting our results on the rates of expected Wasserstein distance between a weighted empirical sampler and the true tilted distribution, we first state the result which is obtained by applying the iid empirical sampler rates in \cite{fournier2013rateconvergencewassersteindistance} to our setting. In a certain sense, which can be formalized, this is the ``golden" standard one wishes to compare with. 

\begin{proposition}[\cite{fournier2013rateconvergencewassersteindistance} applied to $\mu_\theta$]\label{fgappltomt}
Let $p>0$. Assume that $M_q(\mu_\theta)<\infty$ for some $q>p$.
There exists a constant $C$ depending only on $p,d,q$ such that, for all $N\geq 1$, $p < d/2$ and $d > \frac{qp}{q-p}$ we have $$
\E_{\mu_\theta}\left(\cT_p(\mu_{\theta, N}^{e},\mu_\theta)\right) \leq
CM_q(\mu_\theta)^{\frac pq} [N^{-p/d}+ N^{-1/2}]
    $$
\end{proposition}

We now present our results. We believe that these results are of independent interest. The first result assumes moment conditions on $\mu_{2\theta}$, while the second assumes that $\|\rg(\rx)\| \leq g^m$.

\begin{theorem}\label{thm:WassWtd}
For measures $\mu_{\theta}, \mu_{N, \theta}$, as defined in \eqref{true} and \eqref{emp} respectively, and $d > \frac{qp}{q - p}$ with $q > p$, we have 
\begin{equation*}\E[\cT_p(\mu_{N, \theta}, \mu_{\theta})] \leq 
CM_q(\mu_{2\theta})^{\frac pq}C_w[ N^{-p/d} + W_2 N^{-1/2}]
,\end{equation*}

where $W_k,C_w$ are as in \eqref{Wk}, \eqref{Cw} and $C$ is a constant independent of $\mu$ and $N$.
\end{theorem}

\begin{theorem}\label{thm:WassWtd2} 
x   For measures $\mu_{\theta}, \mu_{N, \theta}$ defined in \eqref{true} and \eqref{emp} respectively, with $\|g(\rx)\| \leq g_{max}$ where $\rx \sim \mu$, $d > \frac{qp}{q-p}$ with $q > p$, we have 
$$
\E[\cT_p(\mu_{N, \theta}, \mu_{\theta})] \leq C M_q(\mu_{\theta})^{\frac pq}\cdot (V\cdot N^{-\frac{p}{d}} + N^{-1/2})
$$
where $V = (\exp(\|\theta\|g_{max})M(2\theta))/M(\theta)$.
\end{theorem}

We make a few observations contrasting the above result with Proposition~\ref{fgappltomt}. Observe that the bound in Theorem~\ref{thm:WassWtd2} is tantamount, up to constants, to the $N^{-\frac 12}$ bound seen in Proposition~\ref{fgappltomt}. This justifies its optimality, since Proposition~\ref{fgappltomt} attains the expected rate in the vanilla empirical setting.

On the other hand, in Theorem~\ref{thm:WassWtd} the constant depends on the $q^{\text{th}}$ moment of $\mu_{2\theta}$ instead of $\mu_\theta$, which appears to be non-optimal. However, in Appendix~\ref{proofThm45}, we show why the $\theta \to 2\theta$ blowup is \emph{necessary}.

To control the Wasserstein upper bounds, one needs to control the absolute difference in the measure assigned by $\mu_\theta$ and $\mu_{n, \theta}$ to specific Borel sets, in expectation. One starts by proving that one can suitably normalize the measure before performing a $2^n$ level coupling. 

Then, the following lemma will be applied to that coupling, and, it would be used to derive the required upperbound. 
The complete proof is in Appendix ~\ref{proofThm45}. We refer the reader to \cite{fournier2013rateconvergencewassersteindistance} for more discussion on this coupling. 

\begin{lemma}\label{lemma:WassWtdEBC} For a Borel set $A$,
\begin{equation*}
    \E[|\mu_\theta - \mu_{n, \theta}|(A)] \leq
    \frac{1}{\sqrt{n}}\sqrt{C_w}\left[\sqrt{\mu_{2\theta}(A)} 
        + 
        \mu_\theta(A)
        \right],
\end{equation*}
where $C_w$ is as in \eqref{Cw}.
\end{lemma}

The following corollary of Theorem~\ref{thm:WassWtd2} asserts asymptotic accuracy of reweighed sampling (i.e. as $N,\theta \to \infty$) under suitable conditions.
\begin{corollary}\label{cor:WassConv}
Under the assumptions of Theorem~\ref{thm:WassWtd2}, if for a sequence $\{(N, \theta_N)\}$, 
$$
M_q(\mu_{\theta})^{\frac pq}\cdot (V\cdot N^{-\frac{p}{d}} + N^{-1/2})\to 0,
$$
with the data dimension $d > \frac{qp}{(q - p)}$ and $q > p$, then,
$$
\E \cT_p(\mu_{N, \theta_N}, \mu_{\theta_N}) \to 0.
$$
\end{corollary}

A similar result can be given using Theorem~\ref{thm:WassWtd2}, which we omit from consideration. Unfortunately, the quantities $C_w$ given by \eqref{Cw} and $V$  exhibit exponential growth as $\theta \to \infty$, limiting its usage in algorithmic settings. The resolution of this is left to future work.

\section{ON THE ACCURACY OF DIFFUSION}\label{sec:diff}

In this section, we state Theorem~\ref{thm:DiffWorks} and briefly explain the idea behind the proof. We begin with the diffusion mechanism in Section~\ref{sec:ddpm}, follow it up with the key perturbation bound Proposition~\ref{thm:diffbds} in Section~\ref{sec:diffbds}, and finally state the theorem in Section~\ref{sec:fr}.

\subsection{A BRIEF EXPOSITION ON DIFFUSION}\label{sec:ddpm}

Diffusion sampling \cite{dhariwal2021diffusion_gans_neurips,kingma2021variational_diffusion_neurips,sohl2015deep_unsupervised_icml} is a state-of-the-art algorithm which, given samples from an unknown distribution $\rx$, outputs more samples from the same distribution. We refer the reader to \cite{cao2022survey_diffusion, yang2022diffusion_survey,croitoru2022diffusion_vision} for some surveys on diffusion sampling and its benefits.

For completeness, we shall now explain briefly the process of diffusion sampling. Let $\mu$ be a measure whose samples are provided as an initial input to a diffusion sampler. 

Diffusion begins with the forward process, which is an Ornstein-Uhlenbeck process started at time $0$ from the measure $\mu$. More specifically, let $\{\rx_t\}_{t \geq 0}$ be the solution of the stochastic differential equation (SDE)
\begin{equation}\label{fwd}
d\rx_t = -\eta \rx_t dt + \sqrt{2} \sigma d \rb_t
\end{equation}
on $[0,T]$, where $\{\rb_t\}_{t \geq 0}$ is a Brownian motion, $\eta>0$ is some drift parameter, $\sigma$ is a noise parameter and $T>0$ is some fixed time endpoint which must equal $+\infty$ for diffusion to be theoretically accurate. It is well known that $\rx_t$ converges to a centered Gaussian in distribution as $t \to \infty$. In particular, the solution to this equation is given by 
\begin{equation}\label{soln}
\rx_t = e^{-\eta t} \rx_0 + \sigma \int_0^t e^{-\eta(t-s)} db_s.
\end{equation}

Therefore, assuming that $\rx_T$ is close to the limiting normal, one wishes to reverse the process conducted above starting from normal samples, to obtain new samples of $\rx_t$. This can, in fact, be done. We can show that $\rx^{\leftarrow}_t = \{\rx_{T-t}\}_{t \geq 0}$ satisfies the SDE
\begin{equation}\label{bwd}
d\rx^{\leftarrow}_t = \{\eta \rx^{\leftarrow}_t + 2 \nabla \ln q_{T-t}(\rx^{\leftarrow}_t)\}dt + \sqrt{2}d\rb^{\leftarrow}_t,
\end{equation}
where $\rb^{\leftarrow}_t$ is the reversed Brownian motion. Here, $\nabla \ln q_{t}$ is referred to as the \emph{score} function, and note that this must be ascertained before the reverse process begins. That is done using a score-matching algorithm using a neural network; we refer the reader to \cite{hyvarinen2005score_matching,vincent2011score_dae} for the details. Note that diffusion, as performed above, suffers from three inaccuracies, namely inefficient score estimation, insufficiency of $T$, and any discretization that may be used to simulate the solutions of the SDEs. We shall not comment on the latter two inefficiencies, and \emph{assume they are sufficiently dealt with throughout this paper}: however, the accuracy of score estimation will be our next focus.

\subsection{EFFICIENCY OF SCORE MATCHING UNDER PERTURBED INPUT}\label{sec:diffbds}

The inefficiency of score matching is captured using a loss function. For each $t \in [0,T]$, let $s(\rx_t,t)$ be an estimate of the score $\nabla \ln q_t(\rx_t)$ at time $t$ and space point $\rx_t$ ($s(\rx_t,t)$ can be the output of any statistical procedure or machine learning model that attempts to learn the score). Let $f_t(\rx_t) = s(\rx_t,t) - \nabla \ln q_t(\rx_t)$ be the error, and define \begin{equation}\label{loss}
l(\mu) = \frac{1}{T} \int_{0}^T \E\|f_t(\rx_t)\|^2 dt.
\end{equation}
We shall now make a Lipschitz assumption on $f_t(x)$, which is rather common in the literature.
\begin{assumption}\label{ass:lip}
There is a real-valued function $L_t, t \in [0,T]$ such that $f_t(x)$ is $L_t$-Lipschitz for each $t \in [0,T]$, and $$
C_{\eta} = \frac 1T \int_{0}^T L_{t}^2 e^{-2 \eta t} d t < \infty.
$$
\end{assumption}

The above assumption is inspired from \cite{chen2023samplingeasylearningscore} and \cite{accumulation_score_error_iclr2026}. One can also show that by assuming the score network to have the same Lipschitz constant as that of the true score function, the denoiser error can only decrease. This means by choosing neural-network architecture that behaves appropriately, one could justify this assumption. Also, note, in particular, that we allow $L_{t}$ to be a constant in $t$, something which is seen in \cite{chen2023samplingeasylearningscore} as well, whose result we will need to use. However, $L_{t}$ may very well be unbounded near $0$ or $T$. This adds to the generality of our next main result, which we shall now state and believe is one that is of independent interest. Define   
$$\Delta(\mu, \nu, f) = \E_{t \sim U[0,T]}\left|\E_{\mu}\|f_t(\rx_t)\|^2 - \E_{\nu}\|f_t(\ry_t)\|^2\right|.$$


\begin{proposition}\label{thm:diffbds}
Let $\mu,\nu$ be probability distributions and $\varepsilon>0$ be such that $l(\nu) \leq \varepsilon^2$. Let $\{\rx_t\}_{t \in [0,T]}$ and $\{\ry_t\}_{t \in [0,T]}$ be the forward processes specified by \eqref{fwd} with $\rx_0 = \mu$ and $\ry_0  = \nu$ respectively. Then,
\begin{enumerate}
\item we have 
\begin{equation}\label{eq:one}
 \Delta(\mu, \nu, f) \leq C_{\eta} \cW_2^2(\mu,\nu) + 2 \sqrt{C_{\eta}} \cW_2(\mu,\nu)\varepsilon.
\end{equation}
where $C_{\eta}$ is as in Assumption~\ref{ass:lip}.
\item If, in addition, $\mu$ and $\nu$ are concentrated on the set $\{\|x\| \leq M\}$, then
\begin{equation}\label{eq:two}
\Delta(\mu, \nu, f) \leq  (2C_{\eta}M+2 \sqrt{C_{\eta}}\varepsilon)\cW_2(\mu,\nu).
\end{equation}
\item Furthermore, under the boundedness assumption of the previous point, if $K = 2MC_{\eta}$ then
\begin{equation}\label{eq:three}
\Delta(\mu, \nu, f) \leq K\cW_1(\mu,\nu) + 2 \sqrt{K\varepsilon} \sqrt{\cW_1(\mu,\nu)}.
\end{equation}
\end{enumerate}
\end{proposition}

Informally, probability measures close in the Wasserstein distances can be diffusion sampled with high accuracy provided at least one of them can be, as the denoiser error completely determines the performance of sampling via diffusion. Observe that the stated bounds cover both the bounded and unbounded regimes, and proximity in both $\cW_2$ and $\cW_1$, thereby adding to its versatility. Proposition~\ref{thm:diffbds} applied to the tilted setting gives the following corollary.


\begin{corollary}\label{thm:ExpDiffBds}
Following the setup of Proposition~\ref{thm:diffbds}, take $\mu = \mu_{\theta}$ and $\nu = \mu_{N, \theta}$. Let $K = 2MC_\eta$. Then,
\begin{enumerate}
\item we have, 
\begin{equation*}
\E [\Delta(\mu_\theta, \mu_{N, \theta}, f)] \leq 
C_\eta \E\cT_2(\mu_\theta, \mu_{N, \theta}) 
+ 2 \sqrt{C_\eta}\sqrt{\E\cT_2(\mu_\theta, \mu_{N, \theta}})
\end{equation*}
\item If $\mu$ is concentrated on the set $\{\|x\| \leq M\}$, then
\begin{equation}
\E [\Delta(\mu_\theta, \mu_{N, \theta}, f)^2]
\leq  (K+2 \sqrt{C_{\eta}}\varepsilon)^2 \E \cT_2(\mu_\theta,\mu_{N, \theta}).
\end{equation}
\item Furthermore, under the boundedness assumption of the previous point, 
\begin{equation}
\E [\Delta(\mu_\theta, \mu_{N, \theta}, f)] 
\leq K\E \cW_1(\mu_\theta,\mu_{N, \theta})
\quad + 2 \sqrt{K\varepsilon} \sqrt{\E \cW_1(\mu_\theta,\mu_{N, \theta})},
\end{equation}
\end{enumerate}

where the first expectation is on the randomness from the random measure $\mu_{N, \theta}$.
\end{corollary}

The next corollary is trivial from the previous corollary and $\E\sqrt{\rx} \leq \sqrt{\E \rx}$.

\begin{corollary}\label{cor:DelExp}
Set $f_t(x)$ to be the map $x \to s(x, t) - \nabla\log q_t(x)$. Under Assumption ~\ref{ass:lip}, and  $\E\ell(\mu_{N, \theta}) \leq \e^2$. 
Then,
\begin{enumerate}

\item we have, 
\begin{equation*}
    \E_{\mu_\theta}\|f_t(\rx_t)\|^2 \leq \e^2 + 
    C_\eta \E\cT_2(\mu_\theta, \mu_{N, \theta}) 
+ 2 \sqrt{C_\eta}\sqrt{\E\cT_2(\mu_\theta, \mu_{N, \theta}}).
\end{equation*}
\item If $\mu$ is concentrated on the set $\{\|x\| \leq M\}$, 
\begin{equation}
    \E_{\mu_\theta}\|f_t(\rx_t)\|^2 
    \leq \e^2 + 2M C_{\eta}\E \cW_1(\mu,\nu) 
    \quad + 2 \sqrt{2M C_{\eta}\varepsilon} \sqrt{\E \cW_1(\mu_\theta,\mu_{N, \theta})}.
\end{equation}
\end{enumerate}

\end{corollary}
\label{difficultyBounded}

We shall now comment on the above corollary. In particular, the main highlight is that, in light of Theorems~\ref{thm:WassWtd} and \ref{thm:WassWtd2}, we are able to obtain concentration inequalities for $\cW_p(\mu_{\theta}, \mu_{N,\theta})$, $p=1,2$, which strengthen the tractability of the right hand side in the corollary. To stress the importance of this observation, we briefly expand upon the challenge our techniques have surmounted.

There is a rarity, in the literature, of concentration inequalities for $\cW_2(\mu_{\theta}, \mu_{N, \theta})$. Recently, \cite{Lei_2020} did obtain concentration inequalities on $\cW_2(\mu, \mu_N)$ where $\mu_N$ is the i.i.d. empirical estimator of $\mu$. However, this imposed a Log-Sobolev assumption on $\mu$, and had used the Lipschitz nature of $\cW_2(\mu, \mu_N)$ as a function from $\mathcal{X}$ to $\mathbb{R}$.

To explain this further, let $\mu_N$ and $\mu_N'$ be two i.i.d.-sample-generated estimators of $\mu$, arising through the points $(x_i)_i$ and $(x_i')_i$ respectively. The Lipschitz condition is proved using $|\cW_2(\mu_N, \mu) - \cW_2(\mu, \mu_N')| \leq \cW_2(\mu_N, \mu_N')$, and then noticing that the optimal coupling is the trivial coupling that assigns mass $\frac 1N$ to the pair $(x_i, x_i')$.  However, in our setup, it is not possible to obtain a Lipschitz condition for our weighted sampler, even under the stronger Log-Sobolev assumption. This is because we don't obtain a trivial coupling of $\cW_2(\mu_{N, \theta}, \mu_{N, \theta}')$ anymore as the weights could behave arbitrarily badly. 

\subsection{ACCURACY OF DIFFUSION SAMPLING}\label{sec:fr}

We are now ready to state our final theorem, the accuracy of diffusion in generating tilted samples.
\begin{assumption}\label{ass:LossMin}
The loss minimization procedure leads to 
$$\E\ell(\mu_{N, \theta}) \leq \varepsilon^2$$
\end{assumption}
\begin{assumption}\label{ass:Wass}
Either
$$
M_q(\mu_{2\theta})^{\frac pq}C_w[ N^{-p/d} + W_2 N^{-1/2}] 
\leq \delta
$$
or

\begin{enumerate}
\item $\mu$ is concentrated on the set $\{\|x\| \leq M\}$. 
\item $M_q(\mu_{\theta})^{\frac pq}\cdot (V\cdot N^{-\frac{p}{d}} + N^{-1/2}) \leq \delta$
\end{enumerate}

holds. 
\end{assumption}
\begin{theorem}\label{thm:DiffWorks}
    Under Assumption~\ref{ass:lip}, \ref{ass:LossMin}, \ref{ass:Wass}, $d > qp/(q - p)$, and with twist with weight $w = \exp(\theta^T g(\rx))$, the output of the diffusion process has 
    $\mathcal{O}(C_\eta \delta + \sqrt{C_\eta\delta})$
    or
    $\mathcal{O}(C_\eta M \delta + \sqrt{M C_\eta}\sqrt{\delta})$
    TV-error with respect to $\mu_\theta$ for $p = 2$ or $p = 1$ respectively, according to which of the parts of Assumption~\ref{ass:Wass} holds, assuming the other sources of error (discretization and the time the reverse process is run) is small.
\end{theorem}

As remarked at the end of Corollary~\ref{cor:WassConv}, the quantities $V$ and $\cW$ both grow exponentially in $\theta$. Therefore, our sample complexity bounds for the Wasserstein distance are rather weak in nature, but we can confidently isolate these as the problematic terms, since for a large class of bounded random variables with well-behaved tail (e.g. those that belong to the extreme value domain of attraction of a Weibull distribution), the quantities $W_k$ and $M_{q}(\mu_{\theta})$ grow only polynomially in $\theta$.

This completes the discussion of the theoretical aspects of our article. In the next section, we detail the experiments that we have performed to validate the convergence rates, along with comparisons to related literature.

\section{EXPERIMENTS}\label{sec:exp}

In this section, we describe and present results from experiments used to validate our results.

We first describe how to create a base distribution for the experiments. Starting from independent bounded marginals $X_i$'s, we form $Y=AX$ with a randomly chosen matrix $A$, with $A_{ij} \sim U[0, 1]$ and each row of $A$ summing to 1. Thus, each coordinate of $Y$ is a convex combination of the $X_i$'s. This yields a correlated, non-Gaussian target with bounded support. We treat this as the base distribution $p$.

Given a tilting parameter $\theta \in \mathbb{R}^d$, we construct the target reweighed distribution as:
\[
    p_{\theta}(x) \propto e^{\theta^T x}p(x)
\]

Moreover, we use our method for climate modeling in Appendix \ref{app:climate}. Here, we incorporate moment constraints, giving us an exponential twist as shown in Proposition \ref{thm:KLtoExp}.

\subsection{Convergence Rate Validation}
To empirically validate the theoretical bounds derived in Section~\ref{wacre}, specifically Theorem~\ref{thm:WassWtd}, we analyze the convergence behavior of the reweighed empirical estimator $\mu_{N, \theta}$ as the sample size $N$ increases. 

We utilize a 10-dimensional bounded setting in the framework described above. We use the tilting parameter $\theta = (2, \ldots, 2)$ to construct $p_{\theta}$. We vary the sample size $N$ from $10^2$ to $10^5$ and compute the Sliced Wasserstein distance ($\mathcal{W}_2$), which approximates the true Wasserstein distance \cite{Bonneel2015}, between the reweighed samples and the ground truth.

\begin{figure}[ht]
\centering
\fbox{\includegraphics[width=0.95\linewidth]{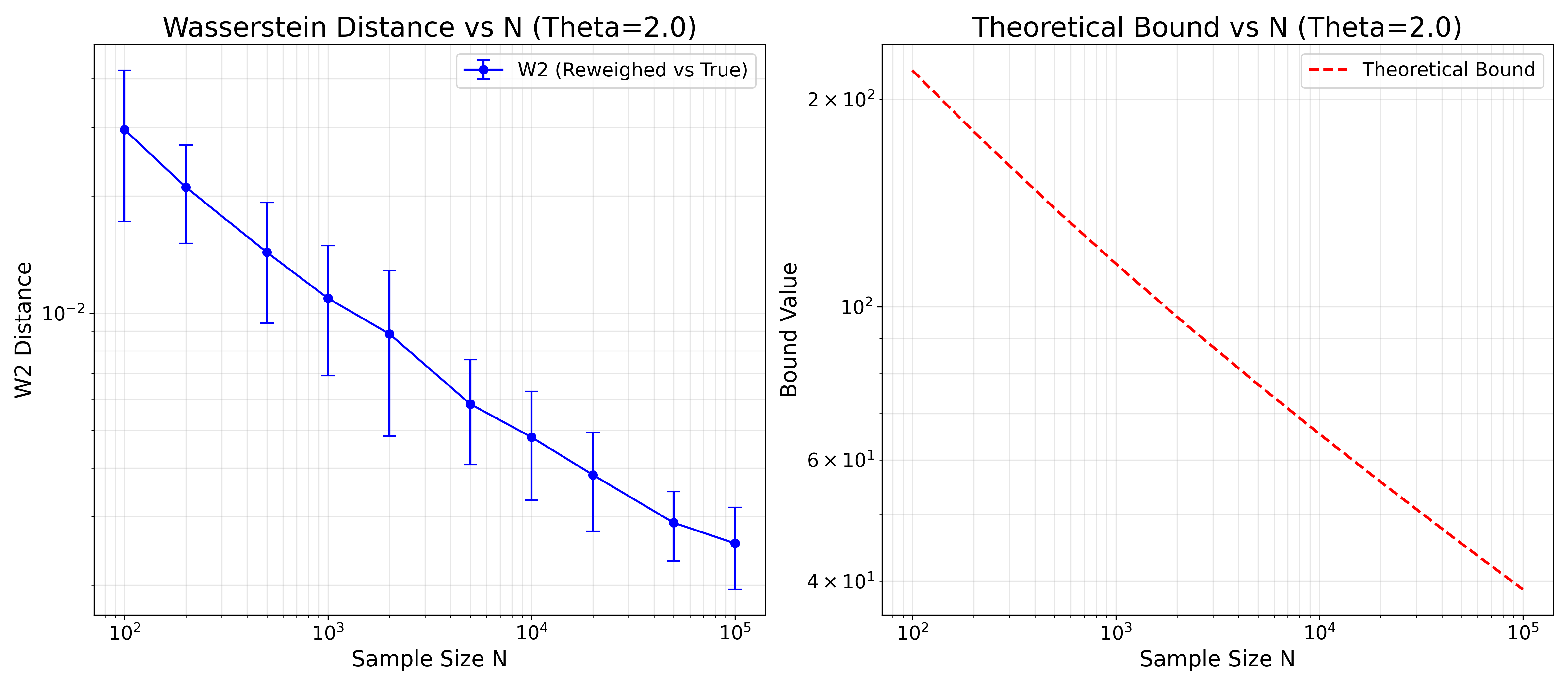}}
\caption{The left plot shows the empirical sliced Wasserstein distance ($\mathcal{W}_2$) between the reweighed estimator and the true tilted distribution. The right plot shows the theoretical bound derived in Theorem~\ref{thm:WassWtd} as a function of sample size $N$. As the theoretical bound vanishes, the empirical error decreases correspondingly.}
\label{fig:convergence}
\end{figure}

The results are presented in Figure~\ref{fig:convergence}. We observe that the empirical error decays monotonically as $N$ increases, closely tracking the slope of the theoretical bound. This confirms that the reweighed estimator converges to the true tilted distribution in the Wasserstein metric as predicted.

\subsection{Bounded, Correlated Target}
We now demonstrate our approach on high-dimensional distributions. We compare four methods:  
i) reweighted sampling,  
ii) reweighted sampling combined with diffusion,  
iii) diffusion posterior sampling (DPS) \cite{chung2024diffusionposteriorsamplinggeneral}, and iv) loss-guided diffusion (LGD-MC) \cite{pmlr-v202-song23k}. 

In this experiment, given a tilting parameter $\theta \in \mathbb{R}^d$, we construct a reweighed distribution using the base distribution $p$ described above:
\[
    p_{\theta}(x) \propto e^{\theta^T x}p(x)
\]
We then use different sampling techniques to sample and report performance across tilt parameters $\theta$. A comparative plot is shown in Figure~\ref{exp2}. For full construction details, refer to Appendix~\ref{app:experiments}.


\begin{figure}[H]
\centering
\fbox{\includegraphics[width=0.47\textwidth]{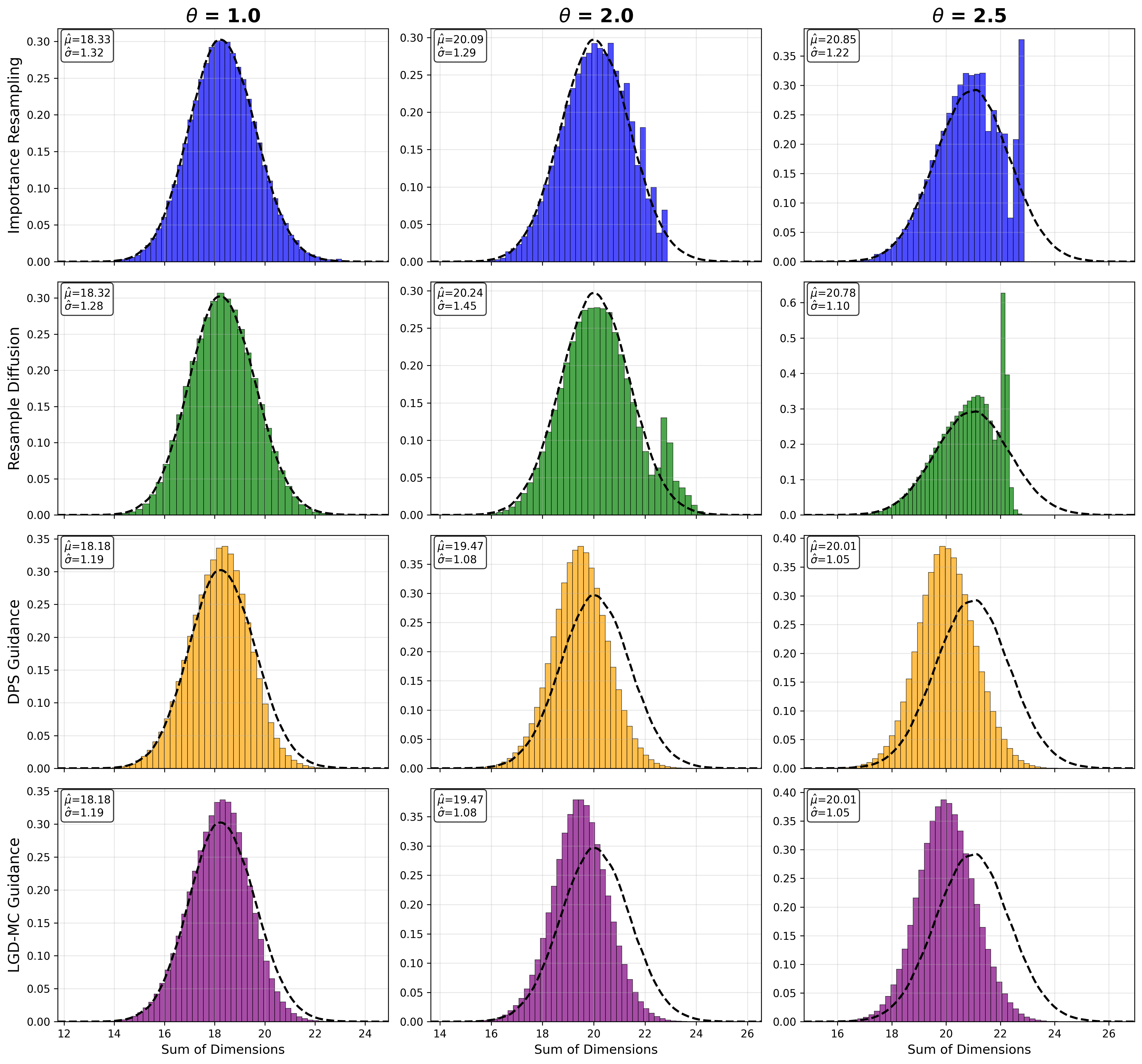}}

\caption{Samples generated by twisting a bounded in 50 dimensions by $\vartheta = \theta \cdot (1, \ldots, 1)$, for $\theta$ = 1.0, 2.0, 2.5 using reweighed sampling, weighted diffusion, DPS, and LGD-MC. We see that our method performs as well as the empirical samples, outperforming the guidance methods.}
\label{exp2}
\end{figure}




\section{CONCLUSION AND FUTURE WORK}\label{conc}

In this work, our goal was to prove that diffusion run on a weighted empirical sampler leads, under suitable conditions, to accurate samples. We first derived upper bounds on the expected Wasserstein distance between the weighted empirical sampler and the true twisted distribution, and secondly used these to obtain bounds on the TV error between the twisted samples and the true twisted distribution. \\

We shall now list some key questions that can be considered for future work. 

\begin{enumerate}[label = (\alph*)]
    \item Theorem~\ref{thm:WassWtd2} shows that the accuracy of the plug-in estimator matches the empirical rate of Proposition~\ref{fgappltomt}, but requires the very restrictive condition $\|g\| \leq g_{\max}$, which often forces boundedness on the support of $\rx$ in practical scenarios. Removing this would constitute a substantial improvement. 
    \item Corollary~\ref{cor:WassConv} provides suitable conditions under which, for the expected Wasserstein distances, asymptotically accurate sampling can be performed. The same question for risk metrics such as value-at-risk are yet to be addressed.
    \item While Theorem~\ref{thm:DiffWorks} upper bounds the sample complexity of diffusion in the expected Wasserstein distance, the question of what lower bounds are possible is completely open in this setting.
    \item We have considered only one particular diffusion mechanism in Section~\ref{sec:diff}. This is often termed as the variance-preserving(or VP) setting. Other settings exist, such as the variance-exploding setup \cite{song2021scorebasedgenerativemodelingstochastic}, flow diffusions\cite{song2021ddim, lipman2023flow_matching_iclr} and rectified flows\cite{liu2023flow_straight_fast_iclr}, and accuracy guarantees for these diffusions in the context of tilted sampling are yet to be established.
    \item Related to (c), the sample complexity is severely affected by the quantities $C_w$ and $V$ given in Corollary~\ref{cor:WassConv}, since these grow exponentially in $\theta$. Controlling them and reducing the sample complexity to a polynomial size would also be a significant improvement. We anticipate that this is the case.
\end{enumerate}

\bibliography{ddpm_based_tilting}
\bibliographystyle{plain}

\clearpage

\begin{appendix}

\section{PROOFS}\label{app:proofs}
In this section, we prove all the results in our article.
\subsection{TILTING AS MINIMIZATION OF ENTROPIC DIVERGENCES}\label{sec:entdiv}
We begin by expanding upon our discussion of tilting, and presenting more results which enhance its applicability. We refer to \cite{csiszar2008axiomatic} for details on all these results.

The first such result is an expanded version of Proposition~\ref{thm:KLtoExp}. 

\begin{proposition}[KL $\Rightarrow$ exponential tilt]\label{thm:KLtoExpapp}
Let $\Phi$ be a random variable on a measure space $(\Omega, \mathcal{F})$ and suppose that $Q$ is a probability measure such that $\E_Q[\exp(\Phi)]<\infty$. For any probability measure $P \ll Q$ such that $\E_P[\phi]<\infty$, consider
\[
\mathcal{J}_{\mathrm{KL}}(P)=\mathbb{E}_P[\Phi]-D_{\mathrm{KL}}(P\|Q),
\]
where $D_{\mathrm{KL}}$ is the $\mathrm{KL}$ divergence defined by
\[
\qquad D_{\mathrm{KL}}(P\|Q)=\int \log\!\frac{dP}{dQ}\,dP.
\]
Then, the unique maximizer of $\mathcal{J}_{\mathrm{KL}}$ is a measure $P^\star$ whose density $h^\star:=\dfrac{dP^\star}{dQ}$ satisfies
\[
\log h^\star(x)=\Phi(x)+\mathrm{const}\quad\text{(a.e.)},
\]
which implies that
\[
\frac{dP^\star}{dQ}(x)=\frac{\exp(\Phi(x))}{\displaystyle\int \exp(\Phi)\,dQ}.
\]
In particular, $P^{\star}$ is a tilt of $Q$.
\end{proposition}
Before proving this result, we state similar results for the R\'enyi and Tsallis entropies.

\begin{proposition}[R\'enyi $\Rightarrow$ power/escort tilt]\label{thm:RenyiToPowerapp}
Consider $Q, (\Omega, \mathcal{F}), \Phi$ as in the previous theorem. Fix $\alpha>0$, $\alpha\neq1$. For $P \ll Q$, let $h=dP/dQ$ and define the \emph{Renyi} entropy
\[
D_{\alpha}(P\|Q)=\frac{1}{\alpha-1}\log\!\Big(\E_{Q}[h^{\alpha}]\Big).
\]
Then, any maximizer $P^\star$ of the functional $\mathcal{J}_{\alpha}(P)=\mathbb{E}_P[\Phi]-D_{\alpha}(P\|Q)$.
possesses a density $h^\star$ satisfying
\[
\frac{\alpha}{\alpha-1}\frac{(h^\star(x))^{\alpha-1}}{\E_{Q}[(h^\star)^\alpha]}
=\Phi(x)-\mu
\]
for some constant $\mu$. Equivalently, after absorbing constants,
\[
h^\star(x)\;\propto\;\big(a + b\,\Phi(x)\big)^{1/(\alpha-1)}.
\]
This corresponds to a power/escort–type tilt (which approaches the exponential tilt as $\alpha\to 1$).
\end{proposition}

Finally, we have another result for the Tsallis entropy.

\begin{proposition}[Tsallis $\Rightarrow$ $q$--exponential tilt]\label{thm:TsallisToqExpapp}
Fix $Q, (\Omega, \mathcal{F}), \Phi$ as in the previous theorems. Let $q>0$, $q\neq 1$ and $P \ll Q$ be such that $\E_P[\Phi]<\infty$. The Tsallis relative entropy is defined by 
\[
D^{\mathrm{Ts}}_{q}(P\|Q)=\frac{1}{q-1}\Big(1-\E_Q[h^q]\Big),\qquad h=\frac{dP}{dQ}.
\]
Any maximizer of the functional $\mathcal{J}_{\mathrm{Ts}}(P)=\mathbb{E}_P[\Phi]-D^{\mathrm{Ts}}_{q}(P\|Q)$.
possesses a density $h^*$ satisfying 
\[
\frac{q}{q-1}h^\star(x)^{\,q-1}=\mu-\Phi(x)
\]
for some constant $\mu$, and hence
\[
h^\star(x)\;\propto\;\big[1+(1-q)\,c\,\Phi(x)\big]^{1/(1-q)}=\exp_q(c\,\Phi(x)),
\]
where $\exp_q$ is the $q$--exponential function (with the usual normalization).
\end{proposition}

We provide sketches of these propositions in the order in which they were stated. Indeed, each one reduces to an appropriate Lagrangian optimization.

\begin{proof}[Proof to Proposition~\ref{thm:KLtoExp}]
Write $h=dP/dQ$. Form the Lagrangian with multiplier $\mu$ for $\int h\,dQ=1$:
\[
\mathcal{L}(h,\mu)=\int \Phi\,h\,dQ-\int h\log h\,dQ-\mu\Big(\int h\,dQ-1\Big).
\]
For any perturbation $\varphi$ with $\int\varphi\,dQ=0$ the first variation vanishes:
\[
0=\frac{d}{d\varepsilon}\Big|_{\varepsilon=0}\mathcal{L}(h+\varepsilon\varphi,\mu)
=\int \varphi(x)\big(\Phi(x)-(\log h(x)+1)-\mu\big)\,dQ(x).
\]
Since this holds for all admissible $\varphi$ the bracket is a.e.\ constant, so
$\log h(x)=\Phi(x)-C$. Exponentiating and normalizing yields the displayed exponential form.
\end{proof}
\begin{proof}[Proof to Proposition~\ref{thm:RenyiToPowerapp}]
Put $A(h):=\int h^\alpha\,dQ$. The Lagrangian with multiplier $\mu$ is
\[
\mathcal{L}(h,\mu)=\int \Phi\,h\,dQ-\frac{1}{\alpha-1}\log A(h)-\mu\Big(\int h\,dQ-1\Big).
\]
For perturbation $\varphi$ with $\int\varphi\,dQ=0$,
\[
0=\frac{d}{d\varepsilon}\Big|_{\varepsilon=0}\mathcal{L}(h+\varepsilon\varphi,\mu)
=\int \varphi(x)\Big(\Phi(x)-\frac{\alpha}{\alpha-1}\frac{h(x)^{\alpha-1}}{A(h)}-\mu\Big)\,dQ(x).
\]
Thus the bracket is a.e.\ constant. Rearranging gives the algebraic relation
displayed above, and solving for $h$ (absorbing multiplicative/additive constants
into $a,b$ and the normalizer) yields the power–law tilt.
\end{proof}
\begin{proof}[Proof to Proposition~\ref{thm:TsallisToqExpapp}]
Up to an additive constant one may write the objective as
\(\int \Phi\,h\,dQ +\tfrac{1}{q-1}\int h^q\,dQ\). Form the Lagrangian
\[
\mathcal{L}(h,\mu)=\int \Phi\,h\,dQ +\frac{1}{q-1}\int h^q\,dQ -\mu\Big(\int h\,dQ-1\Big).
\]
The first variation for perturbations $\varphi$ with $\int\varphi\,dQ=0$ yields
\[
0=\int \varphi(x)\Big(\Phi(x)+\frac{q}{q-1}h(x)^{q-1}-\mu\Big)\,dQ(x).
\]
Thus the bracket is a.e.\ constant, giving the algebraic relation above; solving
and absorbing constants gives the $q$–exponential representation.
\end{proof}

In the next section, we prove the minimaxity of the plug-in estimator. 

\subsection{ON THE MINIMAXITY OF THE PLUG-IN ESTIMATOR}\label{app:sec-minimax}


In this section, we prove Theorem~\ref{thm-minimax}. As remarked earlier, we will modify the proof of the main result in \cite{dvoretzky1956asymptotic}. We start by mentioning some important lemmata that will be used critically in the proofs, and refer the reader to \cite[Chapter 8]{Vaart_1998} for more discussion on these. 
\begin{lemma}[Anderson's Lemma]
    For any bowl-shaped loss function $\ell$ on $\mathbb{R}^k$, every probability measure $M$ on $\mathbb{R}^k$, and every covariance matrix $\Sigma$
    $$
    \int \ell dN(0, \Sigma) \leq \int \ell d[N(0, \Sigma) * M].
    $$
\end{lemma}
\begin{lemma}[Local Asymptotic Minimaxity]\label{app:LAM}
    Let the experiment $(P_\eta : \eta \in \Theta)$ be differentiable in quadratic mean at $\eta$ with non-singular Fisher information matrix $I_\eta$. Let $\Psi$ be differentiable at $\eta$. Let $T_n$ be any estimator sequence in the experiments $(P_\eta^n : \eta \in \mathbb{R}^k)$. Then for any bowl-shaped loss function $\ell$
    \begin{multline*}
    \sup_I \liminf_{n \to \infty} \sup_{h \in I} \E_{\eta + h/\sqrt{n}}\ell \left(\sqrt{n}\left(T_n - \Psi\left(\eta + \frac{h}{\sqrt{n}}\right)\right)\right) 
    \\ \geq \int \ell dN(0,\dot{\Psi_\eta}I_\eta^{-1}\dot{\Psi_\eta}^T).
    \end{multline*}
\end{lemma}
We are now ready to begin the proof.
\begin{proof}[Proof of Theorem~\ref{thm-minimax}]
    Fix $k+1$ points $x_0, x_1, x_2, \cdots, x_k$. Restrict to CDFs supported on just these points. That is, 
    \[
    \mathcal{F}_k = \{F_\pi \text{ CDF } | \ \pi \in \Delta^{k} \subset \mathbb{R}^{k+1}, \ dF_\pi(\{x_i\}) = \pi_i\}
    \]
    It is easy to see that for $\pi \in \operatorname{Int}(\Delta^k)$, the experiment $(F_\pi)$ is differentiable in quadratic mean, with nonsingular information matrix $I_\pi$. This means $F_\pi$ admits a Local Asymptotic Normality (LAN) expansion. Now, consider the vector
    $$\Psi_{\theta}(\pi) = (F_{\pi, \theta}(x_0), F_{\pi, \theta}(x_1), \cdots, F_{\pi, \theta}(x_k))$$
    Restrict to $\mathcal{F}_k \cap \mathcal{F}$ instead. Then, for $\pi$ in the $\operatorname{Int}(\Delta^k)$ with $F_\pi$ in the restricted class, 
    we have that $\Psi_{\theta}(\cdot)$ is differentiable at $\pi$. 
    Consider the bowl-shaped loss function $\ell : \mathbb{R}^{k+1} \to \mathbb{R}$
    $$
    \ell(X) = 1[\|X\|_{\infty} > r]
    $$
    Now, since
    $$
    \sqrt{n} \sup_x \left|T_n(x) - \Psi_\theta(\pi)(x)\right| \geq \sqrt{n} \max_{i = 0}^k \left|T_n(x_i) - \Psi_\theta(\pi)(x_i)\right|
    $$
    Restricting both $T_n$ and $\Psi_\theta$ to the $k+1$ points and using Lemma~\ref{app:LAM}, we get 

    \begin{multline*}
    \sup_I \liminf_{n \to \infty} \sup_{h \in I} P_{F_{\pi + \frac{h}{\sqrt{n}}}}
    \left(\sqrt{n}\left\|T_n - \Psi_\theta\left(\pi + \frac{h}{\sqrt{n}}\right)\right\|_\infty > r\right) 
    \\ \geq P\left(\|N(0,\dot{\Psi_\theta(\pi)}I_\pi^{-1}\dot{\Psi_\theta(\pi)}^T)\|_\infty > r\right).
    \end{multline*}
    
    Drawing $n$ samples from $F_\pi$, let $N_i$ be the count of $x_i$. Let $\hat{\pi} = (N_i/n)_i$. Note that 
    $$ \sqrt{n} [\Psi_\theta(\hat{\pi}) - \Psi_\theta(\pi)]= \sqrt{n}[F_{n, \theta} - F_{\pi, \theta}] \to N(0, \dot{\Psi_\theta(\pi)} I_\pi^{-1}\dot{\Psi_\theta(\pi)}^T)$$

    Then, 
    $$P\left(\|N(0,\dot{\Psi_\theta(\pi)}I_\pi^{-1}\dot{\Psi_\theta(\pi)}^T)\|_\infty > r\right) = \lim_{n \to \infty} P_{F_\pi}(\sqrt{n}\|F_{n, \theta} - F_{\pi, \theta}\|_\infty > r)$$

    Consider the numerator event $A_n$. Since the class $\mathcal{F}$ is compact with respect to the weak convergence topology, the $\lim_n$ and $\sup_{F \in \mathcal{F}}$ can be exchanged. Then, pick an $F^{*} \in \mathcal{F}$ such that 
    $$
    \lim_n P_{F^{*}}\{A_n\} > \lim_n \sup_{F \in \mathcal{F}} P_F\{A_n\} - \varepsilon
    $$

    Grid the class of CDFs using discrete CDFs as discussed above. Consider the following sequence indexed by $k$. Consider sequentially finer griddings such that there exist $\{\pi_k\}$ with $F_{\pi_k} \to F^*$ (as weak convergence). 
    From \cite{iyer2025fundamental}, we know that for $H \in \mathcal{F}$
    $$
    \sqrt{n} (H_{n, \theta} - H_n) \to \mathcal{G}_\theta(H)
    $$
    where $\mathcal{G}_\theta$ is a gaussian random field with the covariance functional 
    $$
    \Cov(\mathcal{G}_\theta(x_1), \mathcal{G}_\theta(x_2)) = \frac{1}{y_3^2}\left(v_{12} - v_{23}\frac{y_1}{y_3} - v_{13}\frac{y_2}{y_3} + v_{33}\frac{y_1y_2}{y_3^2}\right)
    $$
    as defined in \cite{iyer2025fundamental}. 
    Clearly, for $F_{\pi_k}, F^* \in \mathcal{F}$, we get that the  covariance functional converges uniformly as $k \to \infty$, as $\mathcal{F}$ has CDFs concentrated on $\{\|x\| \leq B\}$. 

    Since the function $l^*(\mathcal{G}_\theta(H)) = \sup_x \|\mathcal{G}_\theta(H)(x)\|$ is continuous (\cite{iyer2025fundamental}), we are done by Continuous Mapping Theorem [\cite{Vaart_1998}, Theorem 1.3.6]. 
    
\end{proof}

In the next section, we prove all the results in Section~\ref{wacre}.

\subsection{ON THE WASSERSTEIN ACCURACY OF THE PLUG-IN ESTIMATOR}\label{proofThm45}

In this section, we prove Proposition~\ref{fgappltomt}, and Theorems~\ref{thm:WassWtd} and ~\ref{thm:WassWtd2}, in the given order. Recall $M(\theta)$ from \eqref{mtheta}, and $w_n$ from \eqref{emp}.

As noted in the main body, one needs to understand how far apart are the values assigned by the measures $\mu_{\theta}$ and $\mu_{N,\theta}$ to an arbitrary Borel set $A$ (in expectation). This was captured by Lemma~\ref{lemma:WassWtdEBC}, whose proof we begin with. 

\begin{proof}[Proof of Lemma~\ref{lemma:WassWtdEBC}]
Redefine the probability measure as a fraction of non-normalized measures and their mass on $\mathbb{R}^d$. That is, let $\mu_{\theta} = \tilde{\mu}_{\theta}/M(\theta), \mu_{n, \theta} = \tilde{\mu}_{n, \theta}/w_n^*$ where we recall that $M(\theta) = \E[w]$ and $w_n^* = \frac 1n \sum_i w_i$. Now,

\begin{align*}|\mu_{\theta} - \mu_{n, \theta}|(A)  =& \left|\frac{\tilde{\mu}_\theta}{M(\theta)} - \frac{\tilde{\mu}_{n,\theta}}{w^*_n}\right|(A) \\ \leq & \frac{1}{w^*_n} \left|\tilde{\mu}_{\theta} - \tilde{\mu}_{n, \theta}\right|(A) + \tilde{\mu_\theta} \left|\frac{w^*_n - M(\theta)}{M(\theta)w^*_n}\right|(A) \\
 \leq& \frac 1{w^*_n} \left[|\tilde{\mu}_{\theta} - \tilde{\mu}_{n, \theta}|(A) + \tilde{\mu_\theta} \left|\frac{w^*_n - M(\theta)}{M(\theta)}\right|(A)\right].\end{align*}

Taking the expectation of both sides and applying the Cauchy-Schwarz inequality to the right hand side,
\begin{align}
    &\E|\mu_\theta - \mu_{n, \theta}|(A)   \\ \leq  & \sqrt{\E\left[\frac{1}{(w^*_n)^2}\right]}\left[\sqrt{\E|\tilde{\mu}_{\theta} - \tilde{\mu}_{n, \theta}|^2(A)} + \frac{\tilde{\mu}_\theta}{M(\theta)}\sqrt{\E|w^*_n - M(\theta)|^2}\right] \nonumber\\
     \leq & \sqrt{\E\left[
        \frac{1}{(w^*_n)^2}\right]}\left[\sqrt{\Var(\tilde{\mu}_{n, \theta}(A))} 
        + 
        \frac{\tilde{\mu}_\theta}{M(\theta)}\sqrt{\Var(w^*_n)}
        \right] \nonumber\\ 
    \leq & 
    \sqrt{\E\left[
        \frac{1}{(w^*_n)^2}\right]}\left[\frac{1}{\sqrt{n}}\sqrt{M(2\theta,A)} 
        + 
        \frac{\tilde{\mu}_\theta}{M(\theta)}\frac{1}{\sqrt{n}}\sqrt{M(2\theta)}
        \right]\label{ste1},
\end{align}

where we recall the definition of $M(\theta,A)$ from \eqref{mtheta}. Here, in the second line, we noted that for any centered random variable $\rx$ we have $(\E|\rx|)^2 \leq \Var(\rx)$, and used this observation with $\rx =\tilde{\mu}_{n, \theta}(A)$ and $\rx = w^*_n$ respectively. In the third line, we used the definitions of $\tilde{\mu}_{n, \theta}$ and $w^*_n$. We remark that this is the only place where the structure of the formula $w = \exp(\theta^T g(\rx))$ is being used.

In order to finish the proof, we make the following two observations : $\mu_{2\theta}(A) = \frac{M(2\theta,A)}{M(2\theta)}$ by definition, and 
 \begin{multline*}\E\frac{1}{(w^*_n)^2} = \E\left[\frac{1}{\left(\frac 1n\sum_{i=1}^n \exp(\theta^T \rg(\rx_i))\right)^2}\right] \\ \leq \E\left[\exp\left(-2\theta^T\left[\frac{1}{n} \sum_{i=1}^n \rg(\rx_{i})\right]\right)\right] = M(-2\theta).\end{multline*}
 
Substituting these into \eqref{ste1} completes the proof of the lemma.
\end{proof}

We shall now briefly remark on the tightness of the above proof's techniques. An apparently major loss in the order of the upper bound is due to the blowup from $\mu_\theta(A)\to \mu_{2\theta}(A)$. However, it is actually reasonable to expect that this cannot be removed. To argue this, we first assert the asymptotic normality of the plug-in estimator.

\begin{lemma}\label{lemAN}
$$
    \sqrt{N} \Bigg[\mu_\theta(A) - \mu_{N, \theta}(A)\Bigg] \overset{d}{\to} \mathcal{N}(0, \sigma^2)
$$
Here, $\sigma^2 = \nabla h(\E \ZZ)^T\Sigma\nabla h(\E \ZZ)$, where $h, \Sigma, \ZZ$ are as defined in the proof.
\end{lemma}
\begin{proof}
    Take the vector 
    $$\ZZ_i = \begin{pmatrix} w_i\1_{X_i \in A} \\ w_i \end{pmatrix} = \begin{pmatrix} U_i \\ V_i \end{pmatrix}.$$ 
    
    Define $\bar{\ZZ}_n = \sum_i \ZZ_i/n$. 
    We know the following asymptotic normality holds due to the i.i.d. central limit theorem: 
    $$
    \sqrt{n}(\bar{\ZZ}_n - \E \ZZ) \overset{d}{\to} \mathcal{N}(0, \Sigma).
    $$
    Here,
    $$
    \Sigma = \Var(\ZZ) = 
    \begin{pmatrix} 
    \Var(U) & \Cov(U, V) \\ 
    \Cov(U, V) & \Var(V)
    \end{pmatrix}
    $$
    Define the function $h(x,y) = x/y$. Then by the \textit{Delta method}, one obtains the following asymptotic normality
    $$
    \sqrt{n}(h(\bar{\ZZ}_n) - h(\E \ZZ)) \overset{d}{\to} \mathcal{N}(0, \nabla h(\E \ZZ)^T \Sigma \nabla h(\E \ZZ)).
    $$

    Notice that $h(\bar{\ZZ}_n) = \mu_{N, \theta}(A)$, $h(\E \ZZ) = \mu_\theta(A) = \frac{M(\theta, A)}{M(\theta)}$, $\nabla h(x,y) = \begin{pmatrix} 1/y \\ -x/y^2 \end{pmatrix}$. Then, 

    \begin{equation}\label{Vsquare}
    \sigma^2 = \frac{\Var(U)}{M(\theta)^2} - \frac{2 M(\theta, A)}{M(\theta)^3} \Cov(U, V) + \frac{M(\theta, A)^2}{M(\theta)^4} \Var(V),
    \end{equation}
    where 
    \begin{align*}
    \Var(U) &= M(2\theta, A) - M(\theta, A)^2 \\
    \Var(V) &= M(2\theta) - M(\theta)^2 \\
    \Cov(U, V) &= M(2\theta, A) - M(\theta) M(\theta, A).
    \end{align*}
    Substituting these into \eqref{Vsquare} furnishes an explicit formula for $\sigma$.
\end{proof}

Subsequently, since $\mu_{\theta}$ is associated to a bounded random variable, it is easy to show (see e.g. \cite[Problem 3.4.22]{shiryaev2012problems}) that the central limit theorem can be upgraded to convergence of absolute moments. Considering just the first absolute moment,
$$
\E[\sqrt{N}|\mu_{\theta}(A) - \mu_{N,\theta}(A)|] \to \E[|\mathcal{N}(0,\sigma^2)|] = \sigma \sqrt{\frac 2\pi}.
$$
In particular, this implies that for large $N$, the following approximation holds : 
$$
\E[|\mu_{\theta}(A) - \mu_{N,\theta}(A)|] \approx \sqrt{\frac{2}{\pi}}\frac{\sigma}{\sqrt{n}}.
$$
This explains the appearance of $\mu_{2\theta}$ in Lemma~\ref{lemma:WassWtdEBC}, since it appears in the definition of $\sigma$. Furthermore, it also gives light as to why the Cauchy-Schwarz inequality, used to estimate $w_n$ separately, is in fact the optimal inequality to use in that step. An analogous calculation, which we omit for the sake of brevity, shows that the following approximation holds : 
$$
\E|\tilde{\mu}_{\theta}(A) - \tilde{\mu}_{n, \theta}(A)| \approx \sqrt{\frac{2}{\pi}} \frac{1}{\sqrt{n}} \sqrt{\tilde{\mu}_{2\theta}(A) - \tilde{\mu} _\theta(A)^2}.
$$

We shall now turn to the proof of Theorem~\ref{thm:WassWtd}. As discussed in the main paper, to obtain the nice $M_q(\mu_\theta)^{\frac{p}{q}}$ term we first start by proving a simple scaling result. 
\begin{lemma}\label{lemma:scale}
    Let $\mu_\theta, \mu_{N, \theta}$ be previously defined. Then, 
    $$
        \cT_p(\mu_{N, \theta}, \mu_\theta) = T^{p} \cT_p(\mu_{N, \theta}^{T}, \mu_\theta^{T})
    $$
    where $\mu^T := \mu(TA)$.
\end{lemma}
\begin{proof}
    Let $\pi$ be a coupling of the two measures. Then,
    \begin{align*}
    \cT_p(\mu_{N, \theta}, \mu_\theta) = \inf_{\pi}\E_{(X,Y) \sim \pi} \|X-Y\|^p &= \inf_{\pi}\E_{(X,Y) \sim \pi} T^p \left\|\frac XT-\frac YT\right\|^p \\ &=  T^p \cT_p(\mu_{N, \theta}^{T}, \mu_\theta^{T})
    \end{align*}
\end{proof}

Take $T = \E_{\mu_{2\theta}}\|X\|^q$. Perform the scaling by $T$. Here on, we will assume $\E_{X \sim \mu_{2\theta}} \|X\|^q = 1$. Before proceeding to the next lemma, we will set some notation.

\begin{nota}\label{nota:dfdp}
We have the following as in \cite{fournier2013rateconvergencewassersteindistance}. \begin{enumerate}[label=(\alph*)] 
\item For $\ell\geq 0$, denote by $\cP_\ell$ the natural partition of $(-1,1]^d$ into $2^{d\ell}$
translations of $(-2^{-\ell},2^{-\ell}]^d$.
For two probability measures $\mu,\nu$ on $(-1,1]^d$ and for $p>0$, define
$$
\cD_p(\mu,\nu):= \frac{2^p-1}{2} \sum_{\ell\geq 1} 2^{-p\ell} \sum_{F \in \cP_\ell} |\mu(F)-\nu(F)|,
$$
which is a distance on $\cP((-1,1]^d)$ that is bounded by $1$.

\item Define $B_0:=(-1,1]^d$ and, for
$n\geq 1$, $B_n:=(-2^n,2^n]^d \setminus (-2^{n-1},2^{n-1}]^d$.
For $\mu \in \cP(\R^d)$ and $n\geq 0$, denote by $\cR_{B_n} \mu$ the probability measure
on $(-1,1]^d$ defined as the image of $\mu|_{B_n}/\mu(B_n)$ by the map $x \mapsto x/2^n$.
For two probability measures $\mu,\nu$ on $\R^d$ and for $p>0$, define
$$
\cD_p(\mu,\nu):= \sum_{n\geq 0} 2^{pn} \big( |\mu(B_n)-\nu(B_n)| 
+ (\mu(B_n)\land \nu(B_n)) \cD_p(\cR_{B_n}\mu,\cR_{B_n}\nu) \big).
$$
\end{enumerate}
\end{nota}

We refer the reader to \cite{fournier2013rateconvergencewassersteindistance} for more discussion on $\cD_p$. 

We will now state the following two lemmas from \cite{fournier2013rateconvergencewassersteindistance}. The proofs may be found in that paper.

\begin{lemma}\label{lemma:pairsineq}
Let $d\geq 1$ and  $p>0$.
For all pairs of probability measures $\mu,\nu$ on $\R^d$, $\cT_p(\mu,\nu)\leq \kappa_{p,d}
\cD_p(\mu,\nu)$, with $\kappa_{p,d}:= 2^{p(1+d/2)}(2^p+1)/(2^p-1)$.
\end{lemma}

\begin{lemma}\label{lemma:Dp_pairsineq}
Let $p>0$ and $d\geq 1$. There is a constant $C$, depending only on $p,d$, such that
for all $\mu,\nu \in \cP(\R^d)$, 
$$
\cD_p(\mu,\nu) \leq C \sum_{n\geq 0} 2^{pn} \sum_{\ell\geq 0} 2^{-p\ell} \sum_{F\in\cP_\ell} |\mu(2^nF\cap B_n)
-\nu(2^n F \cap B_n)|
$$
with the notation $2^n F = \{2^n x\;:\; x\in F\}$.
\end{lemma}

Given these lemmas, we can now prove the theorem.
\begin{proof}[Proof of Theorem~\ref{thm:WassWtd}]
We have from Lemma~\ref{lemma:WassWtdEBC}, $$
    \E|\mu_\theta - \mu_{N, \theta}|(A) 
    \leq
    \frac{1}{\sqrt{N}}\sqrt{C_w}\left[\sqrt{\mu_{2\theta}(A)} 
        + 
        \mu_\theta(A)
        \right]. 
$$

Taking $A = 2^n F \cap B_n$ and summing over $F \in \cP_\ell$, using Cauchy-Schwarz and $\#(\cP_\ell) = 2^{dl}$, we have for all $n \geq 0$ that
\begin{multline*}
\sum_{F\in\cP_\ell} \E(|\mu_{N, \theta}(2^nF \cap B_n)-\mu_\theta(2^nF \cap B_n)|) 
\\ \leq  \frac{\sqrt{C_w}}{\sqrt{N}}\left[2^{d\ell/2}\sqrt{\mu_{2\theta}(B_n)} 
        + 
        \mu_\theta(B_n)
        \right].
\end{multline*}

Notice that by \eqref{true} and \eqref{mtheta},
$$
\mu_\theta(B_n) = \E \left[\frac{w}{M(\theta)} \1_{\rx \in B_n}\right] \leq \E \left[\frac{w^2}{M(\theta)^2} \1_{\rx \in B_n}\right]^{1/2} = \frac{\sqrt{M(2 \theta)}}{M(\theta)} \sqrt{\mu_{2\theta}(B_n)}.
$$
Furthermore, we have $$\mu_{2\theta}(B_n) \leq \E_{2\theta}\|\rx\|^q/2^{q(n-1)} = 2^{-q(n-1)}$$

Recall $W_k$ from \eqref{Wk}. By Lemma~\ref{lemma:Dp_pairsineq},
$$
\E(\cD_p(\mu_{N, \theta},\mu_\theta))\leq C \frac{\sqrt{\E[w^{-2}] \E w^2}}{\sqrt{N}}\sum_{n,\ell \geq 0} 2^{pn-p\ell} [2^{d\ell/2} 2^{-qn/2} + W_2 2^{-qn/2}].
$$
Finally, collecting the constants we have
$$
\E(\cD_p(\mu_{N, \theta},\mu_\theta))\leq C \frac{C_w}{\sqrt{N}}\sum_{n\geq 0} 2^{pn} \sum_{\ell\geq 0} 2^{-p\ell} [2^{d\ell/2} \cdot 2^{-qn/2} + W_2 \cdot 2^{-qn/2} ].
$$

At this juncture, we are ready to directly apply results from ~\cite{fournier2013rateconvergencewassersteindistance}. Indeed,  
\begin{align*}
\sum_{\ell\geq 0} 2^{-p\ell} 2^{d\ell/2} (\e/N)^{1/2} 
\leq& C
\left\{\begin{array}{lll}
(\e/N)^{1/2} & \hbox{if} & p>d/2, \\[+3pt]
(\e/N)^{1/2}\log(2+\e N)& \hbox{if} & p=d/2, \\[+3pt]
(\e N)^{-p/d} & \hbox{if} & p\in(0,d/2).
\end{array}\right.
\end{align*}



Applying this yields
$$
\E(\cD_p(\mu_{N, \theta},\mu_\theta))\leq C C_w \left[\sum_{n\geq 0} 2^{pn} 
[2^{-qn(1 - p/d)} N^{-p/d}] + \frac{W_2}{\sqrt{N}}\right]
$$

Note that $pn - qn(1-p/d) = n[\frac{d(p -q) + pq}{d}]$. Since the inequalities $q > p$ and $d > \frac{pq}{q-p}$ hold, the final right hand side is bounded by a constant times the quantity
$$
CC_w[ N^{-p/d} + W_2 N^{-1/2}],
$$
which finishes the result. 
\end{proof}




We are now ready to prove Theorem~\ref{thm:WassWtd2}.

\begin{proof}[Proof of Theorem~\ref{thm:WassWtd2}]

Note if $\|\rx\| \leq B$ then $$w = \exp(\theta^T g(\rx)) \leq \exp(\|\theta\| \ \|g_{\text{max}}\|) = K$$ almost surely. By following the proof of Theorem~\ref{thm:WassWtd}, we have 
\begin{multline*}
\sum_{F\in\cP_\ell} \E(|\mu_{N, \theta}(2^nF \cap B_n)-\mu_\theta(2^nF \cap B_n)|) \\
\leq  \frac{\sqrt{C_w}}{\sqrt{N}}\left[2^{d\ell/2}\sqrt{\mu_{2\theta}(B_n)} 
        + 
        \mu_\theta(B_n)
        \right].
\end{multline*}

In contrast to that proof, however, we employ slightly different bounds on $\mu_{2\theta}(B_n)$ and $\mu_{\theta}(B_n)$. Indeed, we have 
\begin{gather*}\mu_{2\theta}(B_n) = \E[(w^2/M(2 \theta))\1_{\rx \in B_n}] \leq \frac{K M(\theta)}{M(2 \theta)} \mu_\theta(B_n)\end{gather*}
Let $V = \frac{KM(\theta)}{M(2\theta)}$. Then we have by the above bounds that
$$
\E(\cD_p(\mu_{N, \theta},\mu_\theta))\leq C \frac{\sqrt{C_w}}{\sqrt{N}}\sum_{n\geq 0} 2^{pn} \sum_{\ell\geq 0} 2^{-p\ell}
V [2^{d\ell/2} 
2^{-qn/2} + 2^{-qn} ]
$$
From this point, the result follows exactly as in the proof of Theorem~\ref{thm:WassWtd}. 
\end{proof}

We are now ready to prove the theorems in Section~\ref{sec:diff}.

\subsection{ACCURATE SCORE ESTIMATION UNDER PERTURBATION OF INITIAL DISTRIBUTION: PROOF OF PROPOSITION~\ref{thm:diffbds}}\label{proofThm6}

In this section, we prove Proposition~\ref{thm:diffbds}. Before beginning, we reiterate the importance of the Wasserstein distance. Indeed, our proof involves coupling the forward processes $\{x_t\}_{t \in [0,T]},\{y_t\}_{t \in [0,T]}$ using a Wasserstein-optimal coupling of $\mu$ and $\nu$. This makes the results of Section~\ref{sec:reweigh} extremely important, since we require the bounds therein to be on the Wasserstein distance as well.

\begin{proof}[Proof of Proposition~\ref{thm:diffbds}]
In order to proceed with the proof, we ensure that $\mu,\nu$ are probability distributions on a common sample space $(\Omega,\mathcal F)$, and $\{\rb_t\}_{t \geq 0}$ is a Brownian motion on the same space.

We begin with the proof of \eqref{eq:one}. Let $\pi$ denote the joint distribution of $\mu$ and $\nu$. Then, \begin{align}
&\left|\E_{\mu}\|f_t(\rx_t)\|^2 - \E_{\nu}\|f_t(\ry_t)\|^2\right| \nonumber\\
\leq & \left|\E_{\pi}\left[\|f_t(\rx_t)\|^2 -\|f_t(\ry_t)\|^2\right]\right| \nonumber\\
\leq & \E_{\pi}\left|\left[\|f_t(\rx_t)\|^2 -\|f_t(\ry_t)\|^2\right]\right| \nonumber\\
= & \frac 1T \int_0^T \E_{\pi}\left[\left|\left[\|f_t(\rx_t)\|^2 - \|f_t(\ry_t)\|^2\right]\right|\ \  \big|\ \  t = t_0\right] dt_0 \label{st1}
\end{align}
where the equality in the last line, we used the tower rule of conditional expectation.

Observe that for any two arbitrary vectors $u,v \in \R^d$, \begin{align*}
|\|u\|^2 - \|v\|^2| = |(u+v)^T(u-v)| & = |(u-v)^T(u-v) + 2(u-v)^T v|\\ & \leq \|u-v\|^2 + 2 \|u-v\|\|v\|,
\end{align*}
where we used the triangle inequality and the Cauchy-Schwarz inequality in the last step. From this point on, for $t_0 \in [0,T]$ let $E_{t_0,\pi}$ denote the measure $E_{\pi}$ conditioned on $t = t_0$. Applying the previous inequality to the last step of \eqref{st1},
\begin{align}
& \frac 1T \int_0^T \E_{\pi}\left[\left|\left[\|f_t(\rx_t)\|^2 - \|f_t(\ry_t)\|^2\right]\right|\ \  \big|\ \  t = t_0\right] dt_0 \nonumber\\
\leq & \frac 1T \int_0^T \E_{\pi,t_0}\left[\|f_t(\rx_t)-f_t(\ry_t)\|^2\right] + 2 \E_{\pi,t_0}\|f_t(\rx_t) - f_t(\ry_t)\| \|f_t(\ry_t)\| dt_0 \nonumber\\
\leq & \frac 1T \int_0^T \E_{\pi,t_0}\left[L_t^2\|\rx_t-\ry_t\|^2\right] + 2L_t \sqrt{\E_{\pi,t_0}\|\rx_t -\ry_t\|^2} \sqrt{\E_{\pi,t_0}\|f_t(\ry_t)\|^2} dt_0\label{st2},
\end{align}
where in the last line we used the $L_t$ Lipschitz assumption on both terms and the Cauchy-Schwarz inequality on the final term.

So far, we have not made an assumption on $\pi$. Let $\pi$ be a minimizer of \eqref{wass} characterizing $W_2(\mu,\nu)$. If there is no minimizer, we can repeat the above argument with a sequence of approximations $\Pi_n$ to the minimizing value, hence without loss of generality we assume that there is a minimizer.

Since $\rx_t$ and $\ry_t$ satisfy \eqref{soln}, it follows that $\rx_t - \ry_t = e^{-\eta t}(\rx_0-\ry_0)$, and therefore \begin{equation}\label{xtyt}
\E_{\pi,t_0}\|\rx_t -\ry_t\|^2= \E_{\pi,t_0}e^{-2 \eta t}\|\rx_0 -\ry_0\|^2 = e^{-2 \eta t_0} W_2^2(\mu,\nu).
\end{equation}
Applying the above equality in \eqref{st2} we see that 
\begin{align*}
&\frac 1T \int_0^T \E_{\pi,t_0}\left[L_t^2\|\rx_t-\ry_t\|^2\right] + 2 L_t \sqrt{\E_{\pi,t_0}\|\rx_t -\ry_t\|^2}\sqrt{\E_{\pi,t_0}\|f_t(\ry_t)\|^2} dt_0 \\
& \leq W_2^2(\mu,\nu)C_{\eta} + 2\sqrt{C_{\eta}}W_2(\mu,\nu)\sqrt{\frac 1T \int_0^T \left[\E_{\pi,t_0}\|f_t(\ry_t)\|^2\right]} \\
& \leq W_2^2(\mu,\nu)C_{\eta} + 2\varepsilon\sqrt{C_{\eta}} W_2(\mu,\nu)
\end{align*}
where $C_{\eta}$ and $\varepsilon$ are is as in Assumption~\ref{ass:lip} and we used the Cauchy-Schwarz inequality in the last line. This completes the proof of \eqref{eq:one}.

The proof of \eqref{eq:two} follows by noting that if $\rx_0$ and $\ry_0$ have distribution $\mu$ and $\nu$ respectively, then $\|\rx_0-\ry_0\| \leq 2M$ by the triangle inequality. Therefore, for any coupling $\Pi$ of $\mu$ and $\nu$, $$
W_2^2(\mu,\nu) \leq \E_{\pi}[\|\rx_0-\ry_0\|^2] \leq 4M^2.
$$
The inequality $0 \leq W_2(\mu,\nu) \leq 2M$ follows. For any $0 \leq x \leq 2M$, it is clear that $x^2 \leq 2Mx$. Applying this to $x = W_2(\mu,\nu)$, we have $W_2(\mu,\nu)^2 \leq 2MW_2(\mu,\nu)$. Plugging this inequality into \eqref{eq:one} directly yields \eqref{eq:two}.

In order to prove \eqref{eq:three}, we return to the step \eqref{st2} used in the proof of \eqref{eq:one}. Having made a different choice of $\pi$ for that proof, in this proof we choose $\pi$ to be the minimizer of \eqref{wass} used to define $W_1(\mu,\nu)$ (Note : as seen before, without loss of generality one can assume the existence of a minimizer). Then, we have 

\begin{align}
  &\frac 1T \int_0^T \E_{\pi,t_0}\left[L_t^2\|\rx_t-\ry_t\|^2\right] + 2L_t \sqrt{\E_{\pi,t_0}\|\rx_t -\ry_t\|^2} \sqrt{\E_{\pi,t_0}\|f_t(\ry_t)\|^2} dt_0\nonumber\\
 =& \frac 1T \int_0^T \E_{\pi,t_0}\left[L_t^2e^{-2\eta t_0}\|\rx_0-\ry_0\|^2\right] \nonumber \\ +& 2L_t \sqrt{\E_{\pi,t_0}e^{-2 \eta t_0}\|\rx_0 -\ry_0\|^2} \sqrt{\E_{\pi,t_0}\|f_t(\ry_t)\|^2} dt_0.\label{st3}
\end{align}
We will now introduce $W_1(\mu,\nu)$ as follows : observe that $$\E_{\pi}[\|\rx_0-\ry_0\|^2] \leq \E_{\pi}[(\|\rx_0\|+\|\ry_0\|) \|\rx_0-\ry_0\|] \leq 2M \E_{\pi}[\|\rx_0-\ry_0\|] \leq 2M W_{1}(\mu,\nu),$$
where we used the boundedness of the supports of $\mu,\nu$ by $M$. Thus, the right hand side of \eqref{st3} is bounded by
\begin{equation}
2MC_{\eta}W_1(\mu,\nu) + 2\sqrt{2MW_1(\mu,\nu)} \frac 1T \int_0^T L_{t_0}e^{-\eta t_0}\sqrt{\E_{\pi,t_0}\|f_t(\ry_t)\|^2} dt_0\label{st4}.
\end{equation}
The final term is easily bounded by Cauchy Schwarz :
\begin{align*}
&\frac 1T \int_0^T L_{t_0}e^{-\eta t_0}\sqrt{\E_{\pi,t_0}\|f_t(\ry_t)\|^2} dt_0 \\ 
\leq &\sqrt{\frac 1T \int_0^T L^2_{t_0}e^{-2\eta t_0} dt_0} \sqrt{\frac 1T \int_0^T \sqrt{\E_{\pi,t_0}\|f_t(\ry_t)\|^2} dt_0} \\
\leq & \sqrt{C_{\eta}}\sqrt{\varepsilon},
\end{align*}
where we used Assumption~\ref{ass:lip} and a final Cauchy-Schwarz application in the final line. Combining the above with \eqref{st4} completes the proof of \eqref{eq:three} and hence the entire theorem.
\end{proof}

In the next section, we establish our key theorem, Theorem~\ref{thm:DiffWorks} on the accuracy of diffusion.

\subsection{DIFFUSION SAMPLING CAN GENERATE ACCURATE TILTED SAMPLES: PROOF OF THEOREM~\ref{thm:DiffWorks}}\label{proofThm5}

The proof of Theorem~\ref{thm:DiffWorks} is considerably shortened by the observation that once the score-estimation error is bounded by Corollary~\ref{cor:WassConv} and Corollary~\ref{cor:DelExp}, we already know that diffusion is accurate by the main result in \cite{chen2023samplingeasylearningscore}. The brief but necessary details follow.

\begin{proof}[Proof to Theorem \ref{thm:DiffWorks}]
Using Corollary~\ref{cor:WassConv} and assuming that the empirical loss minimization process gives low error, we know that $\E \cT_p(\mu_\theta, \mu_{N, \theta}) \leq \mathcal{O}(\delta)$. By Corollary~\ref{cor:DelExp} we obtain
$$
\E_{\mu_\theta}\|f_t(\rx_t)\|^2 \leq \e^2 + C_\eta \delta + 2 \sqrt{C_\eta}\sqrt{\delta},
$$
and, similarly,
$$
\E_{\mu_\theta}\|f_t(\rx_t)\|^2 \lesssim \e^2 + 2MC_\eta\delta + 2\sqrt{2MC_\eta\e}\sqrt{\delta},
$$

depending on the choice between the $W_1$ or $W_2$ distance. In particular, either result asserts that the denoiser error under the true tilted distribution is small. We are ready to use the following result from 
\cite{chen2023samplingeasylearningscore}, which we state for completeness.
\begin{theorem}[Accuracy of diffusion sampling]\label{thm:ddpm}
    Suppose that Assumptions 1, 2, and 3 hold.
    Let $p_{\theta, T}$ be the output of the DDPM algorithm (Section 2.1) at time $T$, and suppose that the step size $h := T/N$ satisfies $h\lesssim 1/L$, where $L \ge 1$.
    Then, it holds that
    \begin{align*}
        \operatorname{TV}(p_T, q)
        \lesssim {\underbrace{\sqrt{\KL(q \ \| \ \gamma^d)} \exp(-T)}_{\text{convergence of forward process}}} &+ \, \, \, \,  {\underbrace{(L \sqrt{dh}  + L \mathfrak{m}_2 h)\,\sqrt T}_{\text{discretization error}}}
        \, \, \, \, \, \\ &+ {\underbrace{\e_{\text{score}}\sqrt{T}}_{\text{score estimation error}}}\,.
    \end{align*}
\end{theorem}

Using this, we can conclude:
$$
\operatorname{TV}(p_{\theta, T}, \mu_\theta) \leq  \mathcal{O}(\exp(-T)) + \mathcal{O}(\sqrt{h}\sqrt{T}) + \mathcal{O}(\e^2 + \sqrt{\delta}).
$$

Therefore, under the assumption that $T$ is large enough and the discretization parameter $h$ is very small, the proof is complete.
\end{proof}

We have proved all the theoretical results in our paper. The next sections will detail the experiments.

\section{EXPERIMENTS}
\label{app:experiments}

\subsection{Training Algorithm}
\label{sec:reweighted-ddpm}

We train a standard $\epsilon$–prediction DDPM to generate samples from the \emph{tilted} target distribution
\[
\nu(dx)\;\propto\;w(x)\,\mu(dx),\qquad 
w(x)\;=;\exp\!\big(\vtheta^\top \rg(\rx)\big).
\]
We approximate $\nu$ directly by \emph{resampling} the original dataset.

\paragraph{Resampling Strategy.}
Let $\mathcal{D} = \{x_1, \dots, x_N\}$ be the original training set consisting of $N$ samples drawn from the base measure $\mu$. We assign a scalar weight to each sample $x_i \in \mathcal{D}$:
\[
w_i \;=\; \exp\!\big(\vtheta^\top \rg(x_i)\big).
\]
We then construct a new dataset $\mathcal{D}_\nu$ of size $N$ by sampling with replacement from $\mathcal{D}$, where the probability of selecting index $i$ is proportional to its weight:
\begin{equation}
\label{eq:resampling-prob}
p_i \;=\; \frac{w_i}{\sum_{j=1}^N w_j}.
\end{equation}
This procedure essentially implements importance sampling at the dataset level. As $N \to \infty$, the empirical distribution of $\mathcal{D}_\nu$ converges to the target tilted distribution $\nu$.

\paragraph{Standard Training.}
We train the diffusion model on the resampled dataset $\mathcal{D}_\nu$ using the standard, unweighted DDPM objective. Writing the forward noising process as $x_t = \alpha_t x + \sigma_t \varepsilon$, the objective is:
\begin{equation}
\label{eq:standard-objective}
\mathcal L_{\mathrm{simple}}(\phi)
\;=\;
\mathbb E_{t\sim\mathrm{Unif},\,x\sim\mathcal{D}_\nu,\,\varepsilon\sim\mathcal N}
\!\left[\;\big\|\varepsilon-\varepsilon_\phi(x_t,t)\big\|^2\right].
\end{equation}
Because the input $x$ is sampled from $\mathcal{D}_\nu$ (which approximates $\nu$), minimizing Eq.~\eqref{eq:standard-objective} is equivalent to optimizing the score matching objective for the tilted target distribution.

\begin{algorithm}[H]
\caption{Diffusion Training}
\label{alg:reweighted-ddpm}
\begin{algorithmic}[1]
\Require Original dataset $\mathcal{D} = \{x_1, \dots, x_N\}$, tilt parameter $\vtheta$.
\Ensure Trained model parameters $\phi$.

\For{$i = 1$ to $N$}
    \State $w_i \gets \exp\big(\vtheta^\top \rg(x_i)\big)$
\EndFor
\State $p_i \gets w_i / \sum_{j=1}^N w_j$ for all $i$.
\State $\mathcal{D}_\nu \gets$ Sample $N$ items from $\mathcal{D}$ with replacement using probabilities $\{p_i\}$.

\State $\phi \gets \textsc{TrainDiffusion}(\mathcal{D}_\nu)$
\end{algorithmic}
\end{algorithm}

\subsection{Bounded, Correlated Target Construction}

For the bounded setting described in Section~\ref{sec:exp}, we construct a high-dimensional non-Gaussian target distribution with bounded support as follows.

\paragraph{Step 1: Independent bounded marginals.}
We first sample independent random variables
\[
X = (X_1, \ldots, X_d),
\]
where each $X_i$ follows a Beta distribution
\[
X_i \sim \mathrm{Beta}(\alpha_i, \beta_i),
\]
with parameters $\alpha_i \sim \mathrm{Unif}[1, 5]$ and $\beta_i \sim \mathrm{Unif}[1, 5]$.  
This ensures that all coordinates of $X$ are supported on $[0,1]$ while exhibiting varying degrees of skewness and concentration.

\paragraph{Step 2: Inducing correlations.}
To introduce correlations among coordinates, we define
\[
Y = A X,
\]
where $A \in \mathbb{R}^{d \times d}$ is a dense matrix with entries
\[
A_{ij} \sim \mathrm{Unif}[0, 1].
\]
Each column of $A$ is normalized so that
\[
\sum_{i=1}^{d} A_{ij} = 1,
\]
ensuring that every coordinate $Y_i$ is a convex combination of the independent bounded variables $\{X_j\}$. This normalization preserves the overall scaling of the variables so that the magnitude of the exponential tilt induced by $\theta$ remains unaffected.
Consequently, $Y$ remains supported on a bounded support but exhibits non-trivial correlations and non-Gaussian behavior.

\paragraph{Step 3: Twisting and evaluation.}
We perform exponential tilting on the resulting distribution using
$$
\vtheta = \theta \cdot (1, \ldots, 1), \qquad g(x) = x
$$
for several values of $\theta$, and compare three methods:
i) reweighted sampling,  
ii) reweighted sampling combined with diffusion,  
iii) diffusion posterior sampling (DPS) \cite{chung2024diffusionposteriorsamplinggeneral}, and iv) loss-guided diffusion (LGD-MC) \cite{pmlr-v202-song23k}.  

All experiments are conducted in $d = 50$ dimensions with $N = 5 \times 10^6$ samples per setting.

\subsection{Climate Experiment: India Temperature}
\label{app:climate}

We test our methodology via a small climate experiment, by tilting the temperature distribution of India.

\textbf{Data.} Daily temperature over India (ERA5 + CMIP6), $5^\circ\!\times\!5^\circ$ grid, months \textbf{May–June}, years \textbf{1950–2024}. Each day is one sample $x$; $\rg(\rx)$ is the spatial mean.

\textbf{Objective.}
Generate from $P$ (baseline) and from the \emph{moment-constrained} twist
$$
P_\theta \;=\; \arg\min_Q \mathrm{KL}(Q\|P)\ \ \text{s.t.}\ \ \E_Q[\rg(\rx)]=\E_P[\rg(\rx)]+1,
$$
which yields the exponential tilt $dP_\theta/dP \propto w_\theta(x)=\exp(\theta \rg(\rx))$.

\medskip
\textbf{Training.}
We train an $\epsilon$–prediction DDPM on $P$ and $P_\theta$. Architecture and noise schedule are identical across runs.

\begin{figure}[h]
\centering
\fbox{\includegraphics[width=0.875\linewidth]{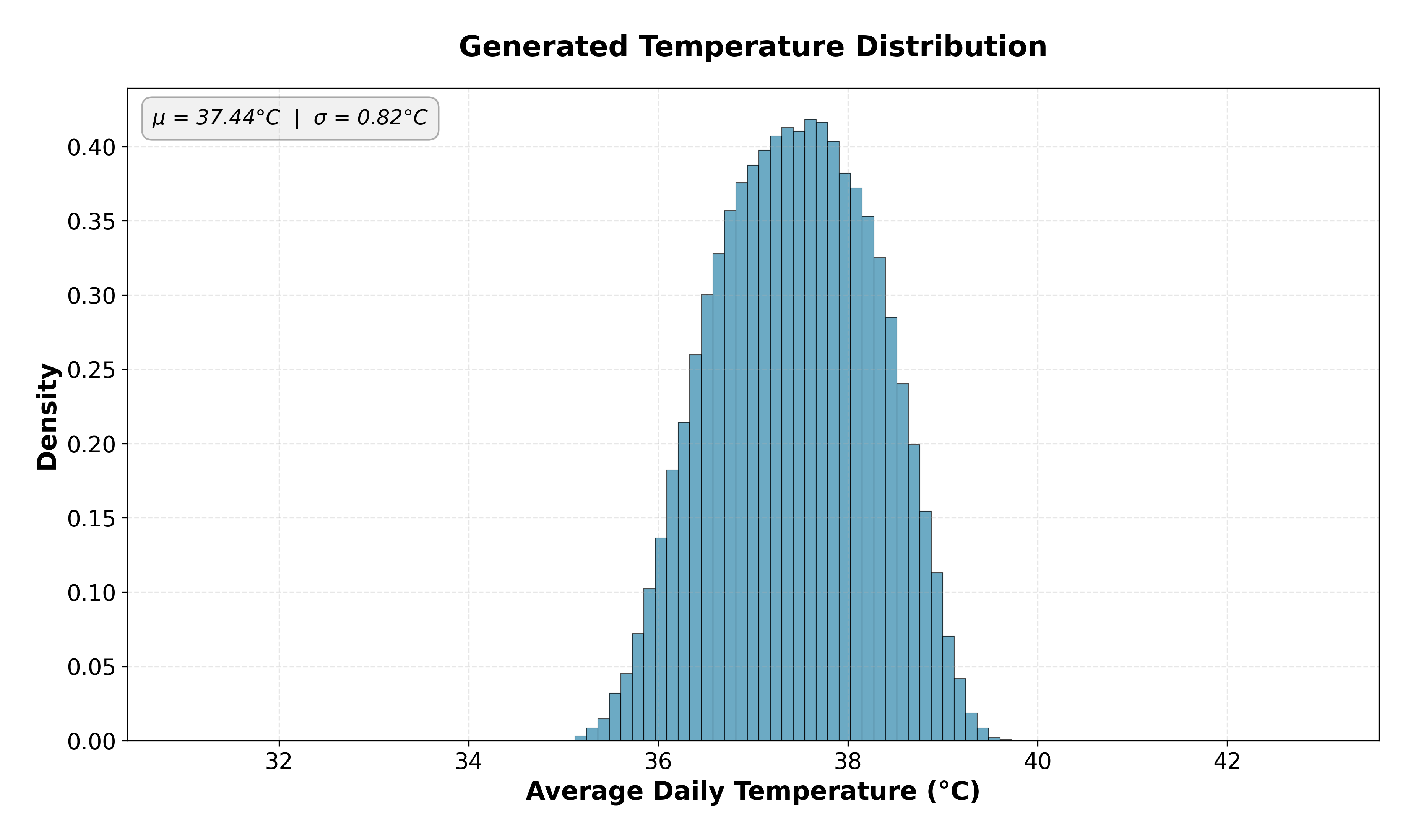}}
\caption{\textbf{DDPM samples from $P$.} Daily temperature fields (May–June, India, $5^\circ\!\times\!5^\circ$).}
\label{fig:ddpm_P}
\end{figure}

\begin{figure}[t]
\centering
\fbox{\includegraphics[width=0.875\linewidth]{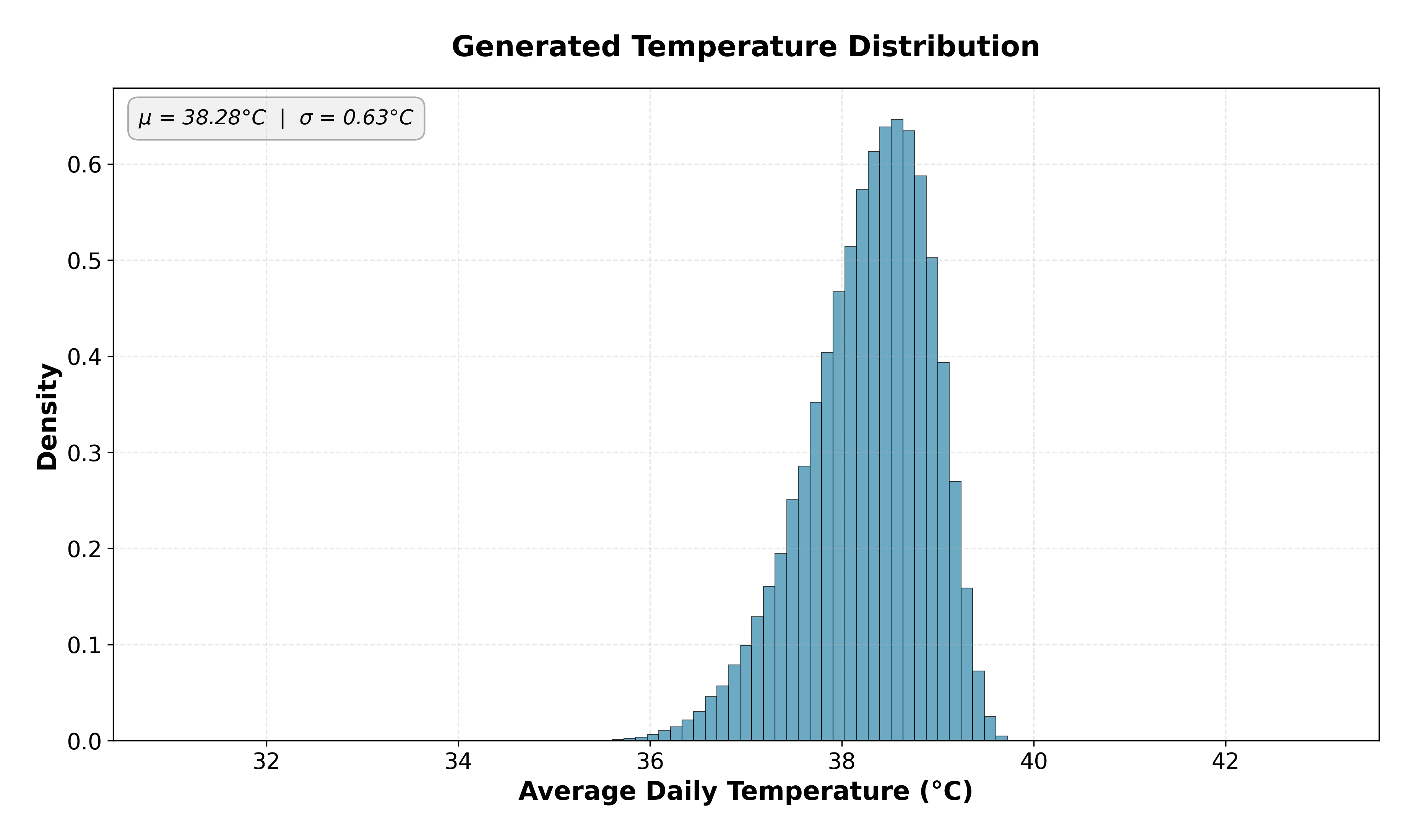}}
\caption{\textbf{DDPM samples from $P_\theta$.} Reweighted training targets the hotter, rarer slice with $\E_{P_\theta}[g]=\E_P[g]+1$.}
\label{fig:ddpm_Ptheta}
\end{figure}

\end{appendix}

\end{document}

%% file: math_commands.tex
\usepackage{amsmath,amsfonts,bm}

\def\eqref#1{equation~\ref{#1}}

\def\1{\bm{1}}

\def\rb{{\textnormal{b}}}

\def\rg{{\textnormal{g}}}

\def\rx{{\textnormal{x}}}
\def\ry{{\textnormal{y}}}

\def\vtheta{{\bm{\theta}}}

\DeclareMathAlphabet{\mathsfit}{\encodingdefault}{\sfdefault}{m}{sl}
\SetMathAlphabet{\mathsfit}{bold}{\encodingdefault}{\sfdefault}{bx}{n}

\newcommand{\R}{\mathbb{R}}

\newcommand{\KL}{D_{\mathrm{KL}}}



%% file: ddpm_based_tilting.bib
@misc{chen2023samplingeasylearningscore,
      title={Sampling is as easy as learning the score: theory for diffusion models with minimal data assumptions}, 
      author={Sitan Chen and Sinho Chewi and Jerry Li and Yuanzhi Li and Adil Salim and Anru R. Zhang},
      year={2023},
      eprint={2209.11215},
      archivePrefix={arXiv},
      primaryClass={cs.LG},
      url={https://arxiv.org/abs/2209.11215}, 
}

@article{Lei_2020,
   title={Convergence and concentration of empirical measures under {W}asserstein distance in unbounded functional spaces},
   volume={26},
   ISSN={1350-7265},
   url={http://dx.doi.org/10.3150/19-BEJ1151},
   DOI={10.3150/19-bej1151},
   number={1},
   journal={Bernoulli},
   publisher={Bernoulli Society for Mathematical Statistics and Probability},
   author={Lei, Jing},
   year={2020},
}

@article{cont2025tail,
  title={Tail-gan: Learning to simulate tail risk scenarios},
  author={Cont, Rama and Cucuringu, Mihai and Xu, Renyuan and Zhang, Chao},
  journal={Management Science},
  year={2025},
  publisher={INFORMS}
}

@article{dvoretzky1956asymptotic,
  title={Asymptotic Minimax Character of the Sample Distribution Function and of the Classical Multinomial Estimator},
  author={Dvoretzky, Aryeh and Kiefer, Jack and Wolfowitz, Jacob},
  journal={The Annals of Mathematical Statistics},
  volume={27},
  number={3},
  pages={642--669},
  year={1956},
  publisher={Institute of Mathematical Statistics},
  doi={10.1214/aoms/1177728174},
}

@book{Vaart_1998, 
place={Cambridge}, 
series={Cambridge Series in Statistical and Probabilistic Mathematics}, 
title={Asymptotic Statistics}, 
publisher={Cambridge University Press}, 
author={Vaart, A. W. van der}, year={1998}, 
collection={Cambridge Series in Statistical and Probabilistic Mathematics}
}

@article{csiszar2008axiomatic,
  title={Axiomatic characterizations of information measures},
  author={Csisz{\'a}r, Imre},
  journal={Entropy},
  volume={10},
  number={3},
  pages={261--273},
  year={2008},
  publisher={Molecular Diversity Preservation International}
}

@article{csiszar1967,
  title={On topology properties of f-divergences.},
  author={Csisz{\'a}r, Imre},
  journal={Studia Scientifica Mathematica Hungerica},
  volume={2},
  number={},
  pages={329–339},
  year={1967},
  publisher={}
}

@article{buchen1996maximum,
  title={The maximum entropy distribution of an asset inferred from option prices},
  author={Buchen, Peter W and Kelly, Michael},
  journal={Journal of Financial and Quantitative Analysis},
  volume={31},
  number={1},
  pages={143--159},
  year={1996},
  publisher={Cambridge University Press}
}

@misc{meucci2010fullyflexibleviewstheory,
      title={Fully Flexible Views: Theory and Practice}, 
      author={Attilio Meucci},
      year={2010},
      eprint={1012.2848},
      archivePrefix={arXiv},
      primaryClass={q-fin.PM},
      url={https://arxiv.org/abs/1012.2848}, 
}

@book{GollRueschendorf2002,
author = {Goll, T. and Rueschendorf, Ludger},
year = {2002},
month = {01},
pages = {},
title = {Minimal distance martingale measures and optimal portfolios consistent with observed market prices},
publisher = "Taylor and Francis"
}

@article{avellaneda1998minimum,
  title={Minimum entropy calibration of asset pricing models, internat},
  author={Avellaneda, M},
  journal={J. Theoret. Appl. Finance},
  volume={1},
  pages={447472},
  year={1998}
}

@article{stutzermichaelnonparametricapproach,
author = {Stutzer, Michael},
title = {A Simple Nonparametric Approach to Derivative Security Valuation},
journal = {The Journal of Finance},
volume = {51},
number = {5},
pages = {1633-1652},
doi = {https://doi.org/10.1111/j.1540-6261.1996.tb05220.x},
url = {https://onlinelibrary.wiley.com/doi/abs/10.1111/j.1540-6261.1996.tb05220.x},
eprint = {https://onlinelibrary.wiley.com/doi/pdf/10.1111/j.1540-6261.1996.tb05220.x},
year = {1996}
}

@misc{fournier2013rateconvergencewassersteindistance,
      title={On the rate of convergence in Wasserstein distance of the empirical measure}, 
      author={Nicolas Fournier and Arnaud Guillin},
      year={2013},
      eprint={1312.2128},
      archivePrefix={arXiv},
      primaryClass={math.PR},
      url={https://arxiv.org/abs/1312.2128}, 
}

@misc{op,
title = {Option Pricing by Esscher Transforms}, 
author = {Gerber, H.U. and Shiu, E.S.W.}, 
year = {1994}, 
journal = {Transactions of the Society of Actuaries}, 
volume = {46}, 
pages = {99-191}
}

@misc{lee2025exponentiallytiltedthermodynamicmaps,
      title={Exponentially Tilted Thermodynamic Maps (expTM): Predicting Phase Transitions Across Temperature, Pressure, and Chemical Potential}, 
      author={Suemin Lee and Ruiyu Wang and Lukas Herron and Pratyush Tiwary},
      year={2025},
      eprint={2503.15080},
      archivePrefix={arXiv},
      primaryClass={cond-mat.stat-mech},
      url={https://arxiv.org/abs/2503.15080}, 
}

@Inbook{Alvo2022,
author="Alvo, Mayer",
title="Exponential Tilting and Its Applications",
bookTitle="Statistical Inference and Machine Learning for Big Data",
year="2022",
publisher="Springer International Publishing",
address="Cham",
pages="171--193",
isbn="978-3-031-06784-6",
doi="10.1007/978-3-031-06784-6_6",
url="https://doi.org/10.1007/978-3-031-06784-6_6"
}

@article{PhysRevE.109.034113,
  title = {Unified perspective on exponential tilt and bridge algorithms for rare trajectories of discrete {M}arkov processes},
  author = {Aguilar, Javier and Gatto, Riccardo},
  journal = {Phys. Rev. E},
  volume = {109},
  issue = {3},
  pages = {034113},
  numpages = {17},
  year = {2024},
  month = {Mar},
  publisher = {American Physical Society},
  doi = {10.1103/PhysRevE.109.034113},
  url = {https://link.aps.org/doi/10.1103/PhysRevE.109.034113}
}

@article{Li2024SEEDS,
  author  = {Li, Lizao and Carver, Robert and Lopez-Gomez, Ignacio and Sha, Fei and Anderson, John R.},
  title   = {Generative emulation of weather forecast ensembles with diffusion models},
  journal = {Science Advances},
  volume  = {10},
  number  = {13},
  pages   = {eadk4489},
  year    = {2024},
  doi     = {10.1126/sciadv.adk4489},
  url     = {https://www.science.org/doi/10.1126/sciadv.adk4489}
}

@article{Ling2024DPDM,
  author  = {Ling, Fenghua and Lu, Zeyu and Luo, Jing-Jia and Bai, Lei and Behera, Swadhin K. and Jin, Dachao and Pan, Baoxiang and Jiang, Huidong and Yamagata, Toshio and others},
  title   = {Diffusion model-based probabilistic downscaling for 180-year {E}ast {A}sian climate reconstruction},
  journal = {npj Climate and Atmospheric Science},
  volume  = {7},
  pages   = {131},
  year    = {2024},
  doi     = {10.1038/s41612-024-00679-1},
  url     = {https://www.nature.com/articles/s41612-024-00679-1}
}

@misc{wang2024proteinconformationgenerationforceguided,
      title={Protein Conformation Generation via Force-Guided {S}{E}(3) Diffusion Models}, 
      author={Yan Wang and Lihao Wang and Yuning Shen and Yiqun Wang and Huizhuo Yuan and Yue Wu and Quanquan Gu},
      year={2024},
      eprint={2403.14088},
      archivePrefix={arXiv},
      primaryClass={q-bio.BM},
      url={https://arxiv.org/abs/2403.14088}, 
}

@misc{song2021scorebasedgenerativemodelingstochastic,
      title={Score-Based Generative Modeling through Stochastic Differential Equations}, 
      author={Yang Song and Jascha Sohl-Dickstein and Diederik P. Kingma and Abhishek Kumar and Stefano Ermon and Ben Poole},
      year={2021},
      eprint={2011.13456},
      archivePrefix={arXiv},
      primaryClass={cs.LG},
      url={https://arxiv.org/abs/2011.13456}, 
}

@article{kong2025robust,
  title={Robust optimization with diffusion models for green security},
  author={Kong, Lingkai and Wang, Haichuan and Pan, Yuqi and Kim, Cheol Woo and Song, Mingxiao and Nguyen, Alayna and Wang, Tonghan and Xu, Haifeng and Tambe, Milind},
  journal={arXiv preprint arXiv:2503.05730},
  year={2025}
}

@article{R,
author = {Ross, N.},
title = {Fundamentals of Stein’s method}, 
journal = {Probab. Surveys},
volume = {8},
pages = {210 - 293},
year = {2011}
}

@article{gao2025wasserstein,
  title        = {Wasserstein Convergence Guarantees for a General Class of Score-Based Generative Models},
  author       = {Gao, Xuefeng and Nguyen, Hoang M. and Zhu, Lingjiong},
  journal      = {Journal of Machine Learning Research},
  volume       = {26},
  number       = {43},
  pages        = {1--54},
  year         = {2025},
  url          = {https://jmlr.org/papers/v26/24-0902.html},
  note         = {Published version with convergence bounds for SGMs in 2-Wasserstein distance},
}

@inproceedings{accumulation_score_error_iclr2026,
  title     = {The Accumulation of Score Estimation Error in Diffusion Models},
  author    = {{Anonymous}},
  booktitle = {International Conference on Learning Representations (ICLR) 2026},
  year      = {2026},
  note      = {Under review as ICLR 2026 submission},
  url       = {https://openreview.net/forum?id=end8EBwFOU},
}

@article{iyer2025fundamental,
  title   = {Fundamental Limits for Weighted Empirical Approximations of Tilted Distributions},
  author  = {Iyer, Sarvesh Ravichandran and Mandal, Himadri and Gupta, Dhruman and Gupta, Rushil and Bandhyopadhyay, Agniv and Bassamboo, Achal and Gupta, Varun and Juneja, Sandeep},
  journal = {CoRR},
  volume  = {abs/2512.23979},
  year    = {2025},
  doi     = {10.48550/arXiv.2512.23979},
  url     = {https://arxiv.org/abs/2512.23979},
}

@book{shiryaev2012problems,
  title={Problems in Probability},
  author={Shiryaev, Albert N},
  year={2012},
  publisher={Springer New York},
  isbn={978-1-4614-3687-4},
  doi={10.1007/978-1-4614-3688-1}
}

@article{benton2023linear_arxiv,
  title   = {Nearly {$d$}-Linear Convergence Bounds for Diffusion Models via Stochastic Localization},
  author  = {Benton, Joe and De Bortoli, Valentin and Doucet, Arnaud and Deligiannidis, George},
  journal = {CoRR},
  volume  = {abs/2308.03686},
  year    = {2023},
  url     = {https://arxiv.org/abs/2308.03686},
}

@article{Owen2000_SafeEffectiveImportanceSampling,
  author       = {Art B. Owen},
  title        = {Safe and Effective Importance Sampling},
  journal      = {Journal of the American Statistical Association},
  year         = {2000},
  volume       = {95},
  number       = {449},
  pages        = {135--143},
  doi          = {10.1080/01621459.2000.10473909}
}

@article{chatterjee2018sample_size_importance,
  title   = {The Sample Size Required in Importance Sampling},
  author  = {Chatterjee, Sourav and Diaconis, Persi},
  journal = {The Annals of Applied Probability},
  volume  = {28},
  number  = {2},
  pages   = {1099--1135},
  year    = {2018},
  doi     = {10.1214/17-AAP1326},
  url     = {https://doi.org/10.1214/17-AAP1326},
}

@book{Rubinstein1981_SimulationAndTheMonteCarloMethod,
  author       = {Reuven Y. Rubinstein},
  title        = {Simulation and the Monte Carlo Method},
  publisher    = {John Wiley \& Sons},
  year         = {1981},
  isbn         = {9780471089179},
  doi          = {10.1002/9780470316511}
}

@InProceedings{pmlr-v202-song23k, title = 	 {Loss-Guided Diffusion Models for Plug-and-Play Controllable Generation}, author =       {Song, Jiaming and Zhang, Qinsheng and Yin, Hongxu and Mardani, Morteza and Liu, Ming-Yu and Kautz, Jan and Chen, Yongxin and Vahdat, Arash}, booktitle = 	 {Proceedings of the 40th International Conference on Machine Learning}, pages = 	 {32483--32498}, year = 	 {2023}, editor = 	 {Krause, Andreas and Brunskill, Emma and Cho, Kyunghyun and Engelhardt, Barbara and Sabato, Sivan and Scarlett, Jonathan}, volume = 	 {202}, series = 	 {Proceedings of Machine Learning Research}, month = 	 {23--29 Jul}, publisher =    {PMLR}, url = 	 {https://proceedings.mlr.press/v202/song23k.html} }

@misc{chung2024diffusionposteriorsamplinggeneral,
      title={Diffusion Posterior Sampling for General Noisy Inverse Problems}, 
      author={Hyungjin Chung and Jeongsol Kim and Michael T. Mccann and Marc L. Klasky and Jong Chul Ye},
      year={2024},
      eprint={2209.14687},
      archivePrefix={arXiv},
      primaryClass={stat.ML},
      url={https://arxiv.org/abs/2209.14687}, 
}

@inproceedings{dhariwal2021diffusion_gans_neurips,
  title     = {Diffusion Models Beat {G}{A}{N}s on Image Synthesis},
  author    = {Dhariwal, Prafulla and Nichol, Alexander Quinn},
  booktitle = {Advances in Neural Information Processing Systems 34},
  editor    = {Marc'Aurelio Ranzato and Alina Beygelzimer and Yann N. Dauphin and Percy Liang and Jennifer Wortman Vaughan},
  pages     = {8780--8794},
  year      = {2021},
  address   = {Virtual Conference},
  publisher = {Curran Associates, Inc.}
}

@inproceedings{kingma2021variational_diffusion_neurips,
  title     = {Variational Diffusion Models},
  author    = {Kingma, Diederik P. and Salimans, Tim and Poole, Ben and Ho, Jonathan},
  booktitle = {Advances in Neural Information Processing Systems 34},
  editor    = {Marc'Aurelio Ranzato and Alina Beygelzimer and Yann N. Dauphin and Percy Liang and Jennifer Wortman Vaughan},
  pages     = {21696--21707},
  year      = {2021},
  address   = {Virtual Conference},
  publisher = {Curran Associates, Inc.},
  url       = {https://proceedings.neurips.cc/paper/2021/file/b578f2a52a0229873fefc2a4b06377fa-Paper.pdf}
}

@inproceedings{sohl2015deep_unsupervised_icml,
  title     = {Deep Unsupervised Learning using Nonequilibrium Thermodynamics},
  author    = {Sohl-Dickstein, Jascha and Weiss, Eric and Maheswaranathan, Niru and Ganguli, Surya},
  booktitle = {Proceedings of the 32nd International Conference on Machine Learning},
  editor    = {Francis Bach and David M. Blei},
  series    = {Proceedings of Machine Learning Research},
  volume    = {37},
  pages     = {2256--2265},
  year      = {2015},
  address   = {Lille, France},
  publisher = {PMLR},
  url       = {https://proceedings.mlr.press/v37/sohl-dickstein15.html}
}

@article{cao2022survey_diffusion,
  title   = {A Survey on Generative Diffusion Models},
  author  = {Cao, Hanqun and Tan, Cheng and Gao, Zhangyang and Chen, Guangyong and Heng, Pheng-Ann and Li, Stan Z.},
  journal = {arXiv e-prints},
  volume  = {abs/2209.02646},
  year    = {2022},
  doi     = {10.48550/arXiv.2209.02646},
  url     = {https://arxiv.org/abs/2209.02646}
}

@article{yang2022diffusion_survey,
  title   = {Diffusion Models: A Comprehensive Survey of Methods and Applications},
  author  = {Yang, Ling and Zhang, Zhilong and Hong, Shenda and Zhang, Wentao and Cui, Bin},
  journal = {arXiv e-prints},
  volume  = {abs/2209.00796},
  year    = {2022},
  doi     = {10.48550/arXiv.2209.00796},
  url     = {https://arxiv.org/abs/2209.00796}
}

@article{croitoru2022diffusion_vision,
  title   = {Diffusion Models in Vision: A Survey},
  author  = {Croitoru, Florinel-Alin and Hondru, Vlad and Ionescu, Radu Tudor and Shah, Mubarak},
  journal = {arXiv e-prints},
  volume  = {abs/2209.04747},
  year    = {2022},
  doi     = {10.48550/arXiv.2209.04747},
  url     = {https://arxiv.org/abs/2209.04747}
}

@article{hyvarinen2005score_matching,
  title   = {Estimation of Non-Normalized Statistical Models by Score Matching},
  author  = {Hyv{\"a}rinen, Aapo},
  journal = {Journal of Machine Learning Research},
  volume  = {6},
  pages   = {695--709},
  year    = {2005},
  url     = {https://www.jmlr.org/papers/v6/hyvarinen05a.html}
}

@article{vincent2011score_dae,
  title   = {A Connection Between Score Matching and Denoising Autoencoders},
  author  = {Vincent, Pascal},
  journal = {Neural Computation},
  volume  = {23},
  number  = {7},
  pages   = {1661--1674},
  year    = {2011},
  doi     = {10.1162/NECO_a_00142},
  url     = {https://direct.mit.edu/neco/article/23/7/1661/7676}
}

@inproceedings{song2021ddim,
  title     = {Denoising {D}iffusion {I}mplicit {M}odels},
  author    = {Song, Jiaming and Meng, Chenlin and Ermon, Stefano},
  booktitle = {Proceedings of the International Conference on Learning Representations (ICLR) 2021},
  year      = {2021},
  url       = {https://openreview.net/forum?id=St1giarCHLP},
}

@inproceedings{lipman2023flow_matching_iclr,
  title     = {Flow Matching for Generative Modeling},
  author    = {Lipman, Yaron and Chen, Ricky T. Q. and Ben-Hamu, Heli and Nickel, Maximilian and Le, Matt},
  booktitle = {International Conference on Learning Representations (ICLR) 2023},
  year      = {2023},
  address   = {Kigali, Rwanda},
  publisher = {OpenReview.net / ICLR},
  url       = {https://openreview.net/forum?id=HxE-MBEKfPd}
}

@inproceedings{liu2023flow_straight_fast_iclr,
  title     = {Flow Straight and Fast: Learning to Generate and Transfer Data with Rectified Flow},
  author    = {Liu, Xingchao and Gong, Chengyue and Liu, Qiang},
  booktitle = {International Conference on Learning Representations (ICLR) 2023},
  year      = {2023},
  address   = {Kigali, Rwanda},
  publisher = {OpenReview.net / ICLR},
  url       = {https://openreview.net/forum?id=xxxxxxx}
}

@Article{Bonneel2015,
author={Bonneel, Nicolas
and Rabin, Julien
and Peyr{\'e}, Gabriel
and Pfister, Hanspeter},
title={Sliced and Radon Wasserstein Barycenters of Measures},
journal={Journal of Mathematical Imaging and Vision},
year={2015},
month={Jan},
day={01},
volume={51},
number={1},
pages={22-45},
issn={1573-7683},
doi={10.1007/s10851-014-0506-3},
url={https://doi.org/10.1007/s10851-014-0506-3}
}
